\documentclass{article}

\usepackage[preprint]{neurips_2026}

\usepackage[utf8]{inputenc} % allow utf-8 input
\usepackage[T1]{fontenc}    % use 8-bit T1 fonts
\usepackage{url}            % simple URL typesetting
\usepackage{booktabs}       % professional-quality tables
\usepackage{amsfonts}       % blackboard math symbols
\usepackage{nicefrac}       % compact symbols for 1/2, etc.
\usepackage{microtype}      % microtypography
\usepackage{xcolor}         % colors
\usepackage{amsmath}
\usepackage{enumitem}
\usepackage{algorithm}
\usepackage{algorithmic}
\definecolor{citeblue}{RGB}{0,102,204}
\definecolor{refred}{RGB}{216, 81, 64}
\usepackage{amsmath,amssymb,amsthm}

\newtheorem{lemma}{Lemma}

\usepackage{graphicx}
\usepackage{subcaption}
\usepackage{multirow}
\usepackage{float}

\usepackage[colorlinks=true,
            linkcolor=refred,
            citecolor=citeblue,
            urlcolor=blue]{hyperref}

\title{Low-Rank Adaptation for Critic Learning in Off-Policy Reinforcement Learning}

\author{%
  \textbf{Yuan Zhuang}\textsuperscript{1}\thanks{Equal contribution.}
  \quad
  \textbf{Yuexin Bian}\textsuperscript{2}\footnotemark[1]
  \quad
  \textbf{Sihong He}\textsuperscript{3}
  \quad
  \textbf{Jie Feng}\textsuperscript{2}
  \quad
  \textbf{Qing Su}\textsuperscript{1}
  \\
  \textbf{Songyang Han}
  \quad
  \textbf{Jonathan Petit}\textsuperscript{4}
  \quad
  \textbf{Shihao Ji}\textsuperscript{1}
  \quad
  \textbf{Yuanyuan Shi}\textsuperscript{2}
  \quad
  \textbf{Fei Miao}\textsuperscript{5}\thanks{Corresponding author: Fei Miao, Zhejiang University (email: feimiao@zju.edu.cn). This work was mostly done while the author was at University of Connecticut; part of the supplementary work was conducted at Zhejiang University. }
  \\[0.8em]
  \textsuperscript{1}University of Connecticut
  \quad
  \textsuperscript{2}University of California San Diego
  \\
  \textsuperscript{3}University of Texas at Arlington
  \quad
  \textsuperscript{4}Qualcomm\quad
    \textsuperscript{5}Zhejiang University
  \\[0.5em]
  {\small\texttt{yuan.2.zhuang@uconn.edu}\quad\texttt{yubian@ucsd.edu}\quad\texttt{sihong.he@uta.edu}\quad\texttt{jif005@ucsd.edu}}
  \\
  {\small\texttt{sq0925@gmail.com}\quad\texttt{songyang418@gmail.com}\quad\texttt{petit@qti.qualcomm.com}}
  \\
  {\small\texttt{shihao.ji@uconn.edu}\quad\texttt{yyshi@ucsd.edu}\quad\texttt{feimiao@zju.edu.cn}}
}

\begin{document}

\maketitle

\begin{abstract}
Scaling critic capacity is a promising direction for improving off-policy reinforcement learning (RL). However, recent work shows that larger critics are prone to overfitting and instability in replay-based bootstrapped training. In this paper, we propose using Low-Rank Adaptation (LoRA) as a structural regularizer for critic learning from scratch. Our approach freezes randomly initialized base matrices and optimizes only the corresponding low-rank adapters, thereby constraining the learned weight changes to low-rank residuals. We evaluate our method on state-based continuous-control tasks, using SAC with SimbaV2 and BRC on DMC-Hard, and FastTD3 with its SimbaV2 variant on IsaacLab. Empirically, LoRA reduces critic loss during training and improves overall policy performance, achieving the best or competitive results on most tasks. Extensive experiments demonstrate that our low-rank updates provide a simple and effective form of structural regularization for critic learning in off-policy RL. The code is available at \url{https://github.com/paulzyzy/LoRA_RL}.

\end{abstract}

\section{Introduction}

Off-policy reinforcement learning (RL) has become a leading paradigm for continuous control because it can reuse past experience through replay, substantially improving sample efficiency over on-policy methods \cite{haarnoja2018soft,fujimoto2018addressing}. At the core of these methods is the \emph{critic}, which estimates action values and provides the learning signal for policy improvement \cite{fujimoto2018addressing}. Since both policy evaluation and policy improvement depend on the quality of the learned value function, critic estimation plays a central role in the stability, scalability, and final performance of off-policy RL~\cite{farebrother2024stop}.

Motivated by this central role, a growing body of work has sought to improve off-policy RL primarily through better \emph{critic design}. Early work focused on reducing estimation bias. Double Q-learning was introduced to mitigate the overestimation induced by the max operator by decoupling action selection from action evaluation \cite{hasselt2010double}, and this idea later became a standard ingredient in deep RL critics \cite{van2016deep}. More recent works have revisited the critic architecture itself. CrossQ showed that carefully incorporating Batch Normalization can significantly improve sample efficiency in off-policy actor-critic learning \cite{bhatt2024crossq}. SimBa proposed an RL backbone with normalization and residual design choices that enable more effective scaling of network capacity \cite{leesimba}. Building on this line, SimbaV2 introduced hyperspherical normalization together with distributional value estimation to stabilize optimization in larger RL critics \cite{lee2025hyperspherical}. Another line of work revisited the critic learning objective. Distributional RL models the full return distribution rather than only its expectation, leading to more informative and stable value learning \cite{bellemare2017distributional}. Building on this perspective,  Farebrother et al.~\cite{farebrother2024stop} showed that replacing value regression with classification can substantially improve the scalability of value learning. 

More recently, several works have highlighted the promise of scaling in continuous-control RL~\cite{palenicek2025scaling, hansen2024tdmpc,leesimba, nauman2024bigger}. TD-MPC2 demonstrated that increasing model size can substantially improve robustness and performance in continuous control \cite{hansen2024tdmpc}. BRO further showed that larger critics can be highly effective when paired with strong regularization and optimistic exploration \cite{nauman2024bigger}. At the same time, recent studies suggested that simply scaling dense RL networks is often insufficient. Although larger critics can improve expressiveness and have contributed to the strong performance of recent RL methods \cite{lee2025hyperspherical}, off-policy critic learning is particularly vulnerable to failure under scale because it relies on bootstrapped targets learned from non-stationary replay data. As a result, larger dense critics can overfit, become harder to optimize, and fail to realize the full benefits of increased capacity \cite{ma2025sparsity}. These findings indicate that successful critic scaling requires not only larger models, but also appropriate structural regularization.

In this work, we follow this line of thinking and propose an adaptive form of structured regularization for off-policy critics. 
Our key idea is to use Low-Rank Adaptation (LoRA)~\cite{hu2022lora} not as a parameter-efficient fine-tuning method for pretrained models, but as a structural regularizer for critic learning from scratch. Specifically, we freeze the randomly initialized dense critic backbone and optimize only low-rank residual updates to its weights. This preserves a dense forward representation through the fixed random backbone, while restricting learning to a low-dimensional update subspace. 
In this sense, our method is closely related to sparsity-based regularization for RL~\cite{ma2025sparsity}. While prior work regularizes the critic by statically setting a subset of weights to zero and inducing sparsity in the forward computation, our method keeps the critic network dense and regularizes learning by restricting weight updates to a low-dimensional subspace.
We instantiate this idea in two commonly adopted architectures, SimbaV2~\cite{lee2025hyperspherical} and BRC~\cite{nauman2025brc}. For SimbaV2, where hyperspherical weight normalization makes direct LoRA integration nontrivial, we develop a compatible LoRA design that preserves the frozen base matrices while updating only the low-rank adapters, ensuring that the resulting effective weights remain on the hyperspherical geometry.

We summarize our main contributions as follows:
\begin{itemize}[leftmargin=1.5em,itemsep=2pt,topsep=2pt]
    \item \textbf{LoRA as structured sparse regularization for critic learning.}\;
    We propose leveraging LoRA as a structural regularizer for off-policy critic learning, where randomly initialized base matrices are frozen and only the corresponding low-rank adapters are optimized. By constraining critic updates to a low-dimensional subspace, our approach mitigates critic overfitting and improves performance on challenging RL tasks. To the best of our knowledge, this is the first systematic study of LoRA as a regularization mechanism for critic learning in off-policy RL.
    % \item \textbf{Hyperspherical weight normalization for LoRA.}\;
    % We develop a LoRA formulation for critic learning that is compatible with SimbaV2's hyperspherical normalization. This is not directly supported by standard LoRA, since applying hyperspherical normalization to the effective weight would also rescale the frozen base matrix. Our design enables frozen-backbone training while preserving the intended geometric structure of the critic weights.
    \item \textbf{Generality across algorithms and architectures.}\;
Extensive experiments demonstrate that the proposed LoRA regularization principle generalizes across multiple off-policy RL settings. Specifically, we consider three representative settings: SAC with SimbaV2~\cite{lee2025hyperspherical} on DMC-Hard tasks, SAC with BRC~\cite{nauman2025brc} on DMC-Hard tasks, and FastTD3 with SimbaV2~\cite{seo2025fasttd3} on IsaacLab robotics tasks. Across these algorithms, architectures, and benchmarks, LoRA achieves the best or competitive results on most tasks.
\end{itemize}

\section{Related Work}

\paragraph{Scaling Model Capacity in Reinforcement Learning.}
While larger models consistently improve supervised and self-supervised learning~\cite{kaplan2020scaling}, scaling in RL has proven far more delicate~\cite{andrychowicz2020matters,bian2026rn}. Naively widening or deepening value-function networks often degrades performance due to overfitting, plasticity loss, and optimization instability~\cite{kumar2021implicit,lyle2022understanding,nikishin2022primacy}. As a result, several recent works have explored architectures and training strategies that make scaling more effective in RL. SpectralNet~\cite{bjorck2021towards} uses spectral normalization to stabilize large critics during training. SimBa~\cite{leesimba} introduces a simplicity-biased residual architecture that enables effective parameter scaling across RL algorithms. BRO~\cite{nauman2024bigger} combines strong $\ell_2$ regularization, optimistic exploration, and high replay ratios to improve continuous-control performance at larger critic scales. SimbaV2~\cite{lee2025hyperspherical} further improves scalability through hyperspherical normalization of features and weights together with distributional critics. FastTD3~\cite{seo2025fasttd3} shows that parallel simulation and large-batch training can make scaled TD3-style agents competitive on challenging humanoid tasks. However, these advances do not fully resolve the overfitting and instability of critic learning under replay-based training. This has motivated a complementary line of work on regularization techniques for critic learning.

\paragraph{Regularization Techniques for Critic Learning.}
In supervised learning, overfitting is commonly mitigated through standard techniques such as weight decay~\cite{krogh1991simple} and dropout~\cite{srivastava2014dropout}. A related perspective comes from parameter-efficient fine-tuning, where constraining optimization to a sparse or low-dimensional subset of parameters can act as an implicit regularizer and often match or even outperform full fine-tuning~\cite{fu2023effectiveness,hu2022lora}. Critic learning in RL shares similarities with supervised learning, but is further complicated by distribution shift arising from replay data and bootstrapped targets. Prior work has therefore explored RL-specific approaches to controlling overfitting. DroQ improves critic stability at high update-to-data ratios through dropout-based Q-functions~\cite{hiraoka2022dropoutqfunctionsdoublyefficient}. More recently, Ma et al.~\cite{ma2025sparsity} show that static network sparsity induced by one-shot random pruning can mitigate optimization pathologies such as plasticity loss and gradient interference in large RL networks, thereby helping unlock the scaling potential of deep RL. Our work is most closely related to this line of research, but instead of imposing a fixed sparse structure, we improve critic robustness through adaptive low-rank updates that regularize how the critic can change during training. We further show that this design yields lower critic loss during training and consistently improves RL performance.

\section{Preliminaries}
\label{sec:preliminaries}

We first review the off-policy RL algorithms and neural network architecture backbones used in our study. In particular, we consider Soft Actor-Critic (SAC)~\cite{haarnoja2018soft} and FastTD3~\cite{seo2025fasttd3} as representative off-policy algorithms, together with recent state-of-art backbones including SimbaV2~\cite{lee2025hyperspherical} and BRC~\cite{nauman2025brc}.
%Our experiments instantiate this template in three settings:SAC with SimbaV2 critics, FastTD3 with a SimbaV2-style critic backbone, and SAC+BRC with BroNet critics.

\subsection{Off-Policy Actor-Critic Algorithms}
\label{sec:algorithms}

\paragraph{Soft Actor-Critic.}
SAC~\cite{haarnoja2018soft} learns a stochastic policy
$\pi_\theta(a\mid s)$ under the maximum-entropy objective
\begin{equation}
J(\pi)=
\mathbb{E}_{\tau\sim\pi}\!
\left[\sum_{t\ge 0}\gamma^t
\bigl(r(s_t,a_t)+\alpha_{\mathrm{ent}}
\mathcal{H}(\pi(\cdot\mid s_t))\bigr)\right],
\label{eq:maxent_obj}
\end{equation}
where $\alpha_{\mathrm{ent}}$ is the entropy temperature. In scalar notation,
the soft Bellman target is
\begin{equation}
y_{\mathrm{SAC}}
=
r+\gamma^n m
\left(
Q_{\bar\phi}^{\mathrm{agg}}(s',a')
-\alpha_{\mathrm{ent}}\log\pi_\theta(a'\mid s')
\right),
\qquad
a'\sim\pi_\theta(\cdot\mid s').
\label{eq:sac_target}
\end{equation}
In Eqs.~\eqref{eq:sac_target} and~\eqref{eq:td3_target},
$n$ denotes the backup horizon and
$r=\sum_{k=0}^{n-1}\gamma^k r_{t+k}$ denotes the accumulated
discounted reward, where $r_{t+k}$ is the immediate reward
at step $t+k$. The next state is $s'=s_{t+n}$, and
$m\in\{0,1\}$ is the bootstrap mask: $m=1$ retains the
bootstrap term and $m=0$ omits it.
The aggregation $Q^{\mathrm{agg}}$ depends on the critic design: SimbaV2
uses either a single critic or clipped double Q-learning depending on the
environment setting, while BRC uses an ensemble average. The SAC actor
minimizes
\begin{equation}
\mathcal{L}_{\pi}(\theta)=
\mathbb{E}_{s\sim\mathcal{D},\,a\sim\pi_\theta}
\left[
\alpha_{\mathrm{ent}}\log\pi_\theta(a\mid s)
-
Q_{\phi}^{\mathrm{agg}}(s,a)
\right].
\label{eq:sac_actor_loss}
\end{equation}
Target critics are updated by Polyak averaging, where $\phi$ denotes the online critic parameters, $\bar{\phi}$ denotes the target critic parameters, and $\tau \in (0,1]$ is the target update rate.
\begin{equation}
\bar\phi \leftarrow (1-\tau)\bar\phi+\tau\phi .
\label{eq:target_ema}
\end{equation}

\paragraph{FastTD3.}
FastTD3~\cite{seo2025fasttd3} is a high-throughput variant of TD3
\cite{fujimoto2018addressing} that combines massively parallel simulation,
large-batch updates, target-policy smoothing, and distributional critics.
It uses a deterministic policy $\mu_\theta$ and constructs the target action
\begin{equation}
\tilde a'
=
\mathrm{clip}\!\left(
\mu_{\bar\theta}(s')
+
\mathrm{clip}(\epsilon,-c,c),
\,a_{\min},a_{\max}
\right),
\qquad
\epsilon\sim\mathcal{N}(0,\sigma^2 I).
\end{equation}
The scalar target is
\begin{equation}
y_{\mathrm{TD3}}
=
r+\gamma^n m\, Q_{\bar\phi}^{\mathrm{agg}}(s',\tilde a').
\label{eq:td3_target}
\end{equation}
The actor minimizes $-\mathbb{E}_{s\sim\mathcal{D}}[
Q_{\phi}^{\mathrm{agg}}(s,\mu_\theta(s))]$ with delayed policy updates.

\subsection{Categorical Critic Learning}
Categorical value learning has recently been widely adopted to improve critic learning in off-policy reinforcement learning~\cite{bellemare2017distributional,farebrother2024stop}. Both SimbaV2 and BRC build on this idea by using categorical value learning for their critics. Instead of directly outputting a scalar value, the critic predicts a categorical distribution over a fixed value support. The scalar $Q$-value can then be recovered as the expectation of this distribution, while the critic is trained using a cross-entropy loss between the predicted and target value distributions.

Let $\mathcal{Z}=\{z_i\}_{i=1}^{N}$ be a fixed discrete value support with $N$ atoms, where
$
z_i=v_{\min}+\frac{i-1}{N-1}(v_{\max}-v_{\min}).
$
A categorical critic predicts a distribution over this support rather than a scalar value. Specifically, for a state-action pair $(s,a)$, it outputs logits $f_\phi(s,a)\in\mathbb{R}^{N}$ and probabilities $p_\phi(s,a)=\mathrm{softmax}(f_\phi(s,a))$. The scalar value estimate is obtained as the expectation
$
Q_\phi(s,a)=\sum_{i=1}^{N} z_i\,p_{\phi,i}(s,a).
$
The target distribution is constructed by applying the Bellman backup to the categorical distribution predicted by the target critic and projecting the shifted distribution back onto $\mathcal{Z}$ using the C51 projection~\cite{bellemare2017distributional}. 
Let $\hat p(s,a)$ denote the projected target distribution. The critic is trained with the cross-entropy loss
$
\mathcal{L}_{Q}(\phi)
=
-\mathbb{E}_{(s,a)\sim\mathcal{D}}
\sum_{i=1}^{N}
\hat p_i(s,a)\log p_{\phi,i}(s,a).
$
%For SAC-style critics, the shifted atom is\(r+\gamma^n m(z_i-\alpha_{\mathrm{ent}}\log\pi_\theta(a'\mid s'))\), while for FastTD3 it is \(r+\gamma^n m z_i\). 
%For BRC, target probabilities are averaged across the critic ensemble before projection. 

\subsection{SimbaV2}
\label{sec:simbav2}

SimbaV2~\cite{lee2025hyperspherical} stabilizes large off-policy critics by
combining hyperspherical feature normalization, post-update weight projection,
LERP residual blocks, and categorical value learning. Given a normalized observation-action input $\bar x_t$,
SimbaV2 appends a constant coordinate and embeds
\begin{equation}
\tilde x_t=\ell_2\text{-Norm}([\bar x_t;c_{\mathrm{shift}}]),
\qquad
h_t^0=\ell_2\text{-Norm}(s_h^0\odot W_h^0\tilde x_t).
\label{eq:simbav2_input}
\end{equation}
Each residual block applies an inverted bottleneck MLP followed by a learned
linear interpolation:
\begin{equation}
\tilde h_t^\ell
=
\ell_2\text{-Norm}\!\left(
W_{h,2}^\ell
\left[
\mathrm{ReLU}\!\left(s_h^\ell\odot W_{h,1}^\ell h_t^\ell\right)
\right]\right),
\qquad
h_t^{\ell+1}
=
\ell_2\text{-Norm}\!\left((\mathbf{1}-\beta^\ell)
h_t^\ell+\beta^\ell\odot\tilde h_t^\ell
\right).
\label{eq:simbav2_block}
\end{equation}
Here $\mathbf{1}\in\mathbb{R}^{d_h}$ denotes the all-ones vector,
$\beta^\ell\in\mathbb{R}^{d_h}$ is a learnable interpolation vector,
$W_{h,1}^\ell\in\mathbb{R}^{4d_h\times d_h}$,
$W_{h,2}^\ell\in\mathbb{R}^{d_h\times 4d_h}$, and
$s_h^\ell\in\mathbb{R}^{4d_h}$ is a learnable scaling vector. The critic head
maps $h_t^L$ to categorical logits.
Each update is followed by a hyperspherical projection of the weight vectors:
\begin{equation}
W\leftarrow \ell_2\text{-Norm}\!\left(W-\eta\nabla_W\mathcal{L}\right).
\label{eq:simbav2_weight_proj}
\end{equation}
% For FastTD3 with SimbaV2, all the SimbaV2 architectural
% components are kept but this post-update weight projection is not applied.

\subsection{BRC}
\label{sec:brc_prelim}
BRC~\cite{nauman2025brc} combines a high-capacity BroNet~\cite{nauman2024bigger} architecture with categorical cross-entropy value learning. Unlike clipped double Q-learning, which uses the minimum of two critics to form the target, BRC maintains two critics and uses their average as the critic estimate. Its critic backbone is a residual MLP with layer normalization $\mathrm{LN}$ and ReLU activations:
\begin{equation}
h^0=\mathrm{ReLU}(\mathrm{LN}(W_{\mathrm{in}}x)),
\qquad
h^{\ell+1}
=
h^\ell+
\mathrm{LN}\!\left(
W_2^\ell\,
\mathrm{ReLU}(\mathrm{LN}(W_1^\ell h^\ell))
\right).
\label{eq:bronet}
\end{equation}
The final linear layer maps the last hidden representation to categorical logits over the value support.

\section{Motivating Example: Low-Rank Updates under Off-Policy Bootstrapping}
\label{sec:motivating_example}

To illustrate when low-rank critic updates can help in off-policy learning, we construct a minimal example that captures three key ingredients of practical critic training: bootstrapped targets, replay-induced distribution shift, and overparameterized value approximation. The example contrasts low-rank adaptation with full-parameter critic learning, showing how restricting the update space can regularize bootstrapped value learning under nonstationary data. Full details are provided in Appendix~\ref{app:toy_example}. Figure~\ref{fig:toy_combined} shows a clear contrast between two training regimes. Under \emph{static regression}, where the critic is trained directly on the true target values, the dense parameterization achieves the lowest final error, indicating that low-rank updates mainly act as a capacity bottleneck in this supervised setting. Under \emph{bootstrapped off-policy TD}, however, the trend reverses: small-rank LoRA attains lower final true-$Q$ error than the dense baseline, with the best result achieved at \(r=1\). Small-rank LoRA also settles to a lower late-stage Bellman-residual plateau. These results suggest that the benefit of low-rank updates in off-policy RL is not reduced model capacity; rather, under bootstrapping and distribution shift, the low-rank constraint acts as an optimization regularizer that can improve critic learning.
\begin{figure}[H]
    \centering
    \includegraphics[width=\linewidth]{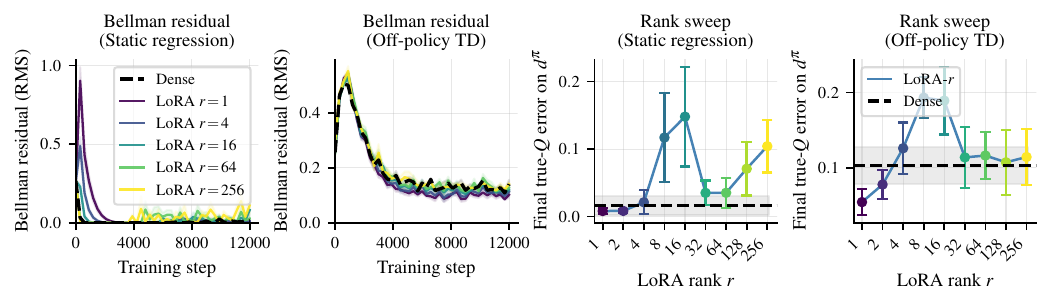}
    \caption{
        A motivating example: \textbf{Left:} Bellman residual (RMS) over training. 
        \textbf{Right: } Final true-$Q$ error on $d^{\pi}$ vs.\
        LoRA rank with 5 seeds; dashed line is the dense
        baseline.
    }
    \label{fig:toy_combined}
\end{figure}

\section{Low-Rank Adaptation for Critic Learning}
\label{sec:method}

In this work, we propose to use Low-Rank Adaptation (LoRA)~\cite{hu2022lora} as a structural constraint for critic learning. Figure~\ref{fig:framework} illustrates how the proposed framework can be integrated with different critic backbones, including SimbaV2 and BRC. Across architectures, the dense forward backbone is preserved, while learning is restricted to low-rank residual updates. This makes the approach architecture-agnostic and easy to adapt to existing off-policy critic designs.

\begin{figure}[h]
    \centering
    \includegraphics[width=\linewidth]{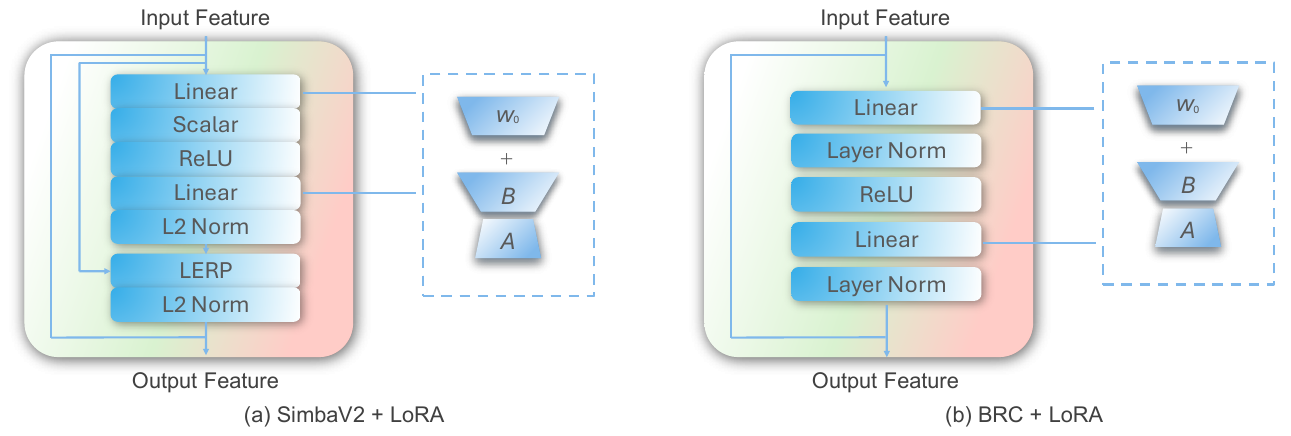}
    \caption{\textbf{Overview of the proposed LoRA-based critic architecture.} For each linear layer in the critic residual block, we freeze the randomly initialized dense base matrix \(W_0\) and train only the low-rank adapters \(A\) and \(B\). Panel (a) represents the LoRA method applied on the SimbaV2 architecture. Panel (b) illustrates the LoRA method for the BRC with BroNet structure.
}
\label{fig:framework}
\end{figure}

\subsection{LoRA Parameterization for Critic Networks}
\label{sec:lora_param}

Consider a generic linear map in a critic residual block,
\[
y = Wx, \qquad W\in\mathbb{R}^{d_{\mathrm{out}}\times d_{\mathrm{in}}}.
\]
In the dense baseline, \(W\) is directly optimized. In our LoRA critic, we instead use
\begin{equation}
\label{eq:lora_eff_weight}
W_{\mathrm{eff}}
=
W_0+\Delta W,
\qquad
\Delta W
=
\frac{\alpha_{\mathrm{lora}}}{r}BA,
\end{equation}
where \(W_0\) is frozen after initialization, \(A\in\mathbb{R}^{r\times d_{\mathrm{in}}}\) and \(B\in\mathbb{R}^{d_{\mathrm{out}}\times r}\) are trainable adapter factors, \(r\) is the LoRA rank, and \(\alpha_{\mathrm{lora}}\) is LoRA scaling factor. Throughout all reported experiments, we set \(\alpha_{\mathrm{lora}}=r\), so that \(\alpha_{\mathrm{lora}}/r=1\). This avoids introducing an additional scaling hyperparameter while retaining the standard LoRA parameterization.

The critic forward pass uses
\begin{equation}
\label{eq:lora_forward}
y = W_{\mathrm{eff}}x
=
\left(W_0+\frac{\alpha_{\mathrm{lora}}}{r}BA\right)x,
\end{equation}
and gradients are applied only to \(A\) and \(B\) for LoRA-modified matrices.

We apply this replacement only to critic residual-block linear maps. For SimbaV2 critics, these are the inverted-bottleneck MLP weights \(\{W_{h,1}^{\ell},W_{h,2}^{\ell}\}_{\ell=0}^{L-1}\) in Eq.~\eqref{eq:simbav2_block}. For BRC, these are the corresponding residual-block linear maps. We keep the actor, input embedding layer, output value head, temperature parameters, and normalization statistics unchanged from the original baseline implementations. This design applies LoRA only to the dominant scaled critic parameters, while avoiding task-specific rank choices for environment-dependent input and output layers.

\paragraph{LoRA with Weight hyperspherical Normalization.}
LoRA can be directly applied to BRC, since its critic backbone does not impose additional constraints on the weight matrices. The SAC+SimbaV2 implementation, however, applies a post-update hyperspherical projection that constrains each weight vector to have unit norm. This creates a compatibility issue for LoRA: directly normalizing the effective weight $W_0+\Delta W$ would also rescale the frozen base matrix $W_0$, violating the intended LoRA parameterization. To address this, we design a LoRA-compatible projection, described in Appendix~\ref{sec:ab_norm}, that enforces the unit-norm constraint on the effective weight while keeping the frozen base matrix unchanged.

For each weight vector, instead of normalizing $w_0+\Delta w$ directly, our projection keeps $w_0$ fixed and solves for a scalar $s$ such that
\begin{equation}\label{eq:solve_rank}
    \|w_0+s\Delta w\|_2=1.
\end{equation}
We initialize each frozen base vector with squared Euclidean norm $\kappa=\|w_0\|_2^2=0.5$, so $\|w_0\|_2=\sqrt{0.5}<1$. After each update, we rescale the LoRA residual while keeping the base fixed to enforce the unit-norm constraint in Eq.~\eqref{eq:solve_rank}. The scalar $s$ is absorbed into the LoRA update.

This projection also affects the LoRA initialization. Standard LoRA initializes one adapter factor to zero, yielding $\Delta W=0$ so that the model initially matches the pretrained backbone. In our setting, however, the SimbaV2-compatible projection requires a nonzero update direction to solve for the scalar $s$ in Eq.~\eqref{eq:solve_rank}. If $\Delta w=0$, this direction is undefined and the effective weight norm cannot be adjusted while keeping $w_0$ fixed.

For the SimbaV2-compatible projection described above, we therefore initialize both adapter factors with small random values:
\begin{equation}
A_{ij}\sim\mathcal{N}\!\left(0,\frac{1}{r}\right),
\qquad
B_{jk}\sim\mathcal{N}\!\left(0,\frac{1}{d_{\mathrm{out}}}\right),
\end{equation}
so that $\Delta W=(\alpha_{\mathrm{lora}}/r)BA$ is nonzero at initialization and provides a valid projection direction. The motivating toy experiment instead uses standard LoRA initialization with $B_i=0$, as detailed in Appendix~\ref{app:toy_example}.

\section{Experiments}
\label{sec:experiments}
In this section, we demonstrate that LoRA-regularized critics reduce overfitting in off-policy reinforcement learning and improve performance across diverse architectures and environments. We evaluate our method in three settings: (i) SAC+SimbaV2~\cite{lee2025hyperspherical} on seven DeepMind Control (DMC)-Hard locomotion tasks~\cite{tassa2018deepmind}; (ii) SAC+BRC~\cite{nauman2025brc} with BroNet~\cite{nauman2024bigger} on the same seven DMC-Hard tasks; and (iii) FastTD3+SimbaV2$^\star$~\cite{seo2025fasttd3} on six IsaacLab robotics tasks~\cite{Mittal_2023}. Here, $^\star$ indicates that the official FastTD3 implementation~\cite{seo2025fasttd3} does not use hyperspherical weight normalization. Detailed experimental settings are provided in Appendix~\ref{ap:hyperparams}, and complete results are reported in Appendix~\ref{ap:experiment}.

\paragraph{Baselines.}
In each setting, we compare our method against the corresponding dense reference implementation and a static sparse baseline. The dense references include \textbf{SimbaV2}~\cite{lee2025hyperspherical}, a state-of-the-art architecture that has been widely adopted in recent reinforcement learning algorithms, and \textbf{BRC}~\cite{nauman2025brc}, a recently proposed Bigger, Regularized, Categorical neural network architecture. The \textbf{Sparse} baselines apply one-shot random pruning~\cite{ma2025sparsity}, with the sparsity level chosen to approximately match the trainable critic-parameter budget of LoRA. This baseline reflects prior evidence that sparsity can unlock improved performance in large critic networks. In contrast, our approach freezes the dense base matrices in the critic residual blocks and trains only the low-rank adapters.

% \begin{figure}[h]
%     \centering
%     \includegraphics[width=1.0\linewidth]{figures/main_figure.png}
%     \caption{\textbf{Performance comparison across FastTD3 on IsaacLab and SAC on DMC tasks.} We report average return versus normalized training steps for three variants built on the SimbaV2 backbone: the full-parameter baseline (SimbaV2), static sparse networks (Sparse), and our proposed LoRA-based critics (LoRA). Results are averaged over 5 seeds, with shaded regions indicating standard deviation.}
%     \label{fig:main_figure}
% \end{figure}

\begin{figure}[h]
    \centering
    \includegraphics[width=1.0\linewidth]{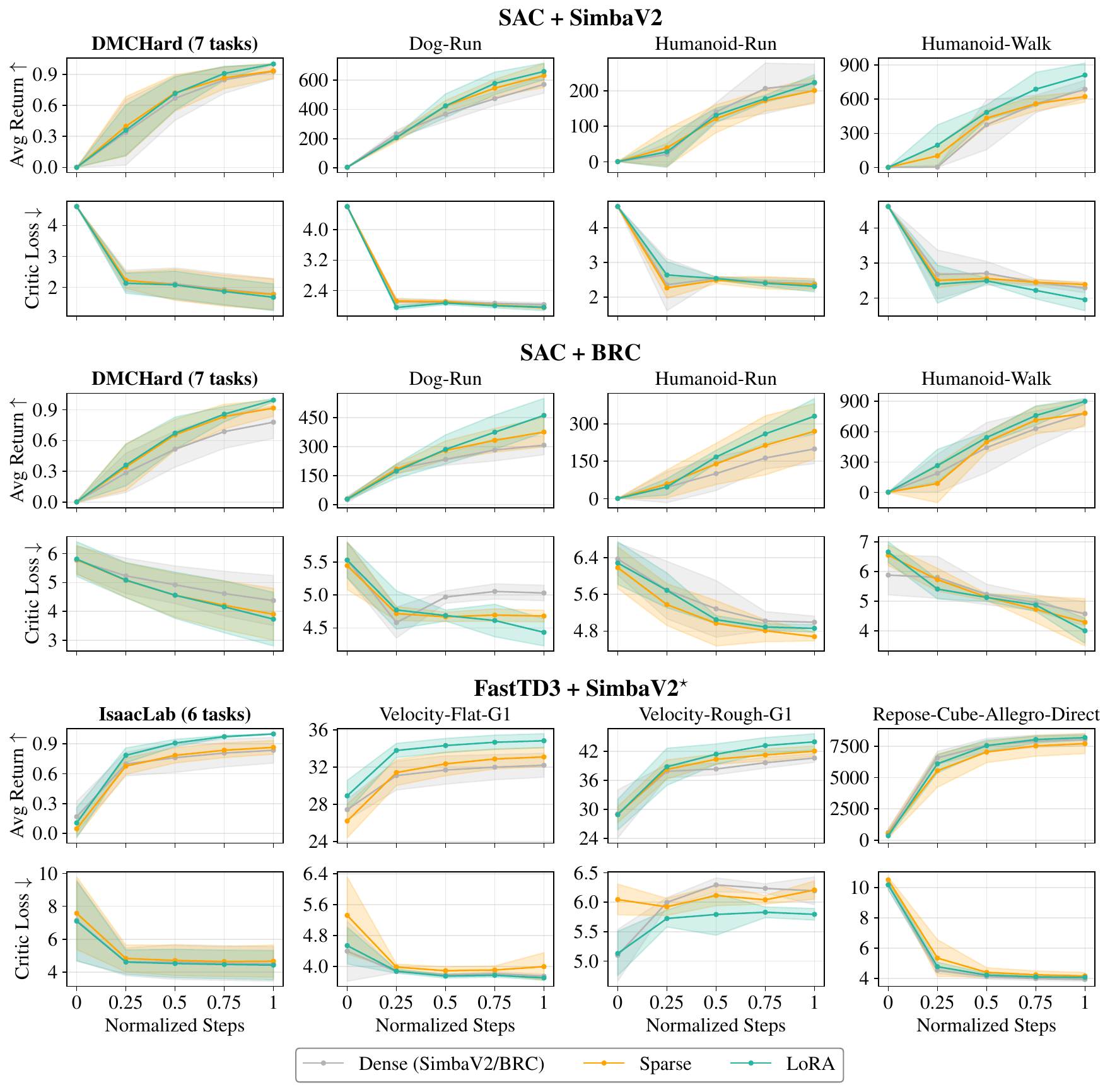}
    \caption{\textbf{Performance comparison across RL algorithms and network architectures.}
We evaluate three settings: SAC+SimbaV2 on DMC-Hard tasks, SAC+BRC on DMC-Hard tasks, and FastTD3+SimbaV2$^\star$ on IsaacLab robotics tasks. For each setting, the leftmost column reports the average normalized return and critic loss across all tasks, while the remaining columns show representative individual tasks. LoRA consistently achieves higher returns than the dense and sparse baselines across most settings and tasks, while maintaining comparable or lower critic loss during training. Results are averaged over 5 seeds, with shaded regions indicating standard deviation. The full per-environment learning curves can be found in Fig~\ref{fig:full_results}.}
    \label{fig:main_results}
\end{figure}

\paragraph{Main results.} 
Figure~\ref{fig:main_results} presents the main performance comparison across algorithms, architectures, and environments. For each setting, we report both policy performance measured by average return, and the critic loss.
For a fair comparison, we choose the LoRA rank and sparse ratio so that the sparse baselines have a comparable number of trainable critic parameters to our method, and the detailed parameter count can be found in Table~\ref{tab:param_counts}. In the SimbaV2 experiments, LoRA uses rank $r$=96, while the sparse baseline uses a sparsity ratio of 0.85. In the BRC experiments, LoRA uses rank $r$=128, while the sparse baseline uses a sparsity ratio of 0.9. Complete numerical results are provided in Appendix~\ref{ap:experiment}, with detailed performance summaries in Tables~\ref{tab:dmc_results}, \ref{tab:brc_bronet_results}, and~\ref{tab:isaaclab_results}.
Across all settings, LoRA achieves the strongest overall performance. Static sparsity already improves over the full-parameter SimbaV2 and BRC baselines, often accelerating early learning and yielding modest gains in final return. This trend is also reflected in the critic-loss curves, where Sparse typically reduces the loss faster than the full-parameter baseline early in training and then maintains a similar or slightly lower loss level.
LoRA provides a clear additional gain beyond static sparsity. Across SAC and FastTD3, LoRA achieves the best return on nearly all aggregate and representative tasks, with particularly clear improvements on IsaacLab and challenging locomotion domains. LoRA also generally yields a stronger critic-loss profile than SimbaV2, BRC, and their Sparse variants, often reaching the lowest loss in later training. We attribute these gains to the regularizing effect of low-rank critic adaptation: under the non-stationary and distribution-shifted data encountered in reinforcement learning, unconstrained critics can overfit to replay samples and learn brittle value estimates. In contrast, LoRA restricts the critic update space, improving fitting stability and enabling more robust value learning.

\paragraph{Held-out TD RMSE and critic training loss.}
Training critic loss alone provides only indirect evidence of reduced overfitting. Unlike conventional supervised learning on a fixed dataset, off-policy critic learning is inherently nonstationary: the replay-buffer distribution, policy, and bootstrapped TD targets evolve throughout training. Consequently, differences in critic loss may also reflect differences in how critics track changing Bellman targets under distribution shift.

To provide more direct evidence, we additionally evaluate the TD RMSE on a held-out set of replay transitions excluded from all gradient updates. This metric measures out-of-sample Bellman consistency rather than performance on the minibatches used for optimization. Figure~\ref{fig:heldout_critic_diagnostics} reports these diagnostics, and Table~\ref{tab:dropout_q_returns} in Appendix~\ref{app:dropout_q} reports the corresponding final returns. On Dog Run, LoRA achieves the lowest held-out TD RMSE and training critic loss, together with the highest mean return. On Velocity-Rough-G1, LoRA achieves the highest mean return and a lower training critic loss than the dense baseline, although it does not achieve the lowest held-out TD RMSE. The relationship between policy performance and either critic loss or held-out TD error is therefore not strictly monotonic across methods and tasks, suggesting that the interplay among critic generalization, regularization, and optimization dynamics may also influence performance.

\begin{figure}[t]
    \centering
    \includegraphics[width=\linewidth]{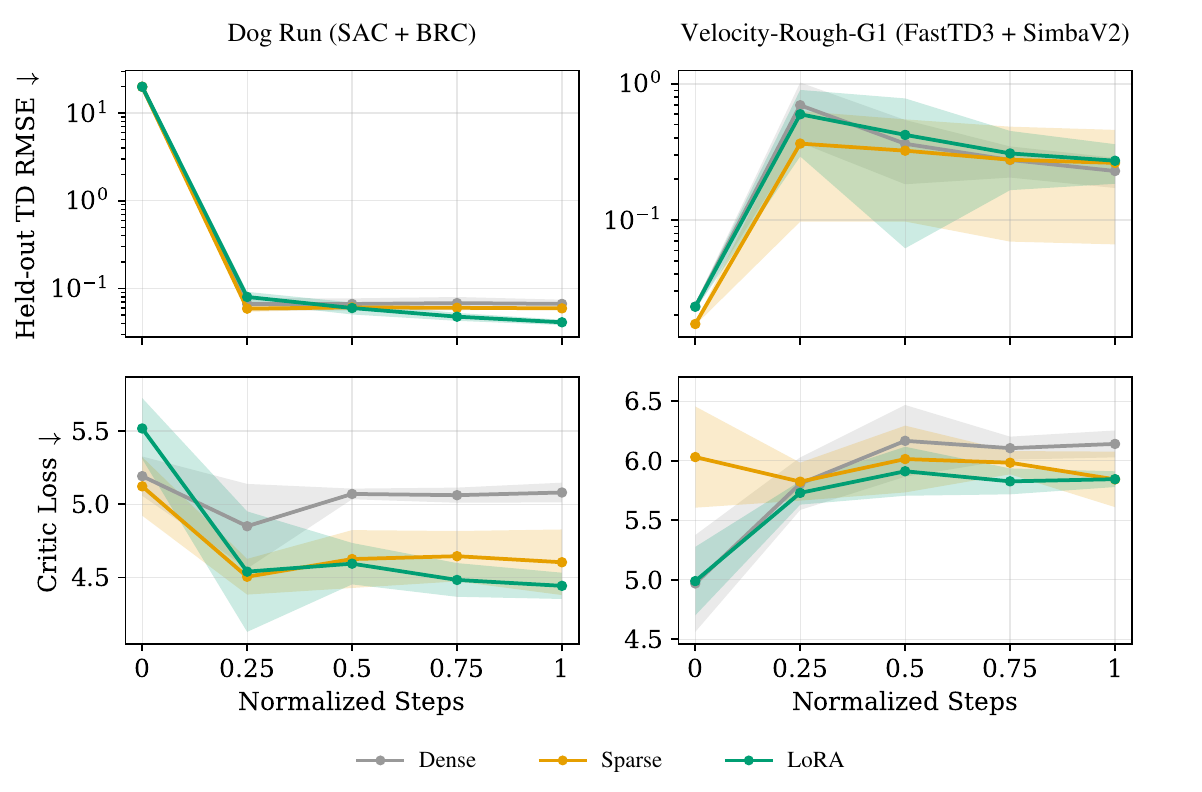}
    \caption{\textbf{Held-out TD RMSE and critic training loss.} Comparison of Dense, Sparse, and LoRA on Dog Run with SAC+BRC (left) and Velocity-Rough-G1 with FastTD3+SimbaV2 (right). The top row shows TD RMSE on held-out replay transitions excluded from gradient updates, with a logarithmic vertical scale. The bottom row shows critic training loss. Curves report the mean over five random seeds (5--9), and shaded regions indicate $\pm$ one standard deviation. The horizontal axis shows normalized training steps.}
    \label{fig:heldout_critic_diagnostics}
\end{figure}

\paragraph{LoRA from Scratch vs. Dense Warm-Up. } Figure~\ref{fig:lora_warm} compares pure LoRA training with hybrid schemes that first train the critic densely for an initial fraction of training and then switch to LoRA. Across all four tasks, pure LoRA remains consistently competitive and often achieves the strongest final return. Dense warm-up can occasionally reduce the early performance gap or provide small transient gains, but these benefits are not consistent and do not reliably improve final performance. In particular, even when one-quarter or one-half of training uses unconstrained dense updates, the resulting hybrid methods do not systematically outperform pure LoRA.
This result suggests that LoRA is not merely useful as a lightweight fine-tuning stage after dense critic training. Instead, its benefit appears to come from constraining critic updates throughout learning. Applying the low-rank constraint from the start shapes the optimization trajectory by limiting the critic’s tendency to fit noisy bootstrapped targets, over-adapt to transient replay-buffer artifacts, or drift under non-stationary data. Once the critic has already undergone a long dense-training phase, these effects may be difficult for a later LoRA phase to correct. Overall, these results support LoRA as a training-time structural regularizer.

\begin{figure}[h]
    \centering
    \includegraphics[width=1.0\linewidth]{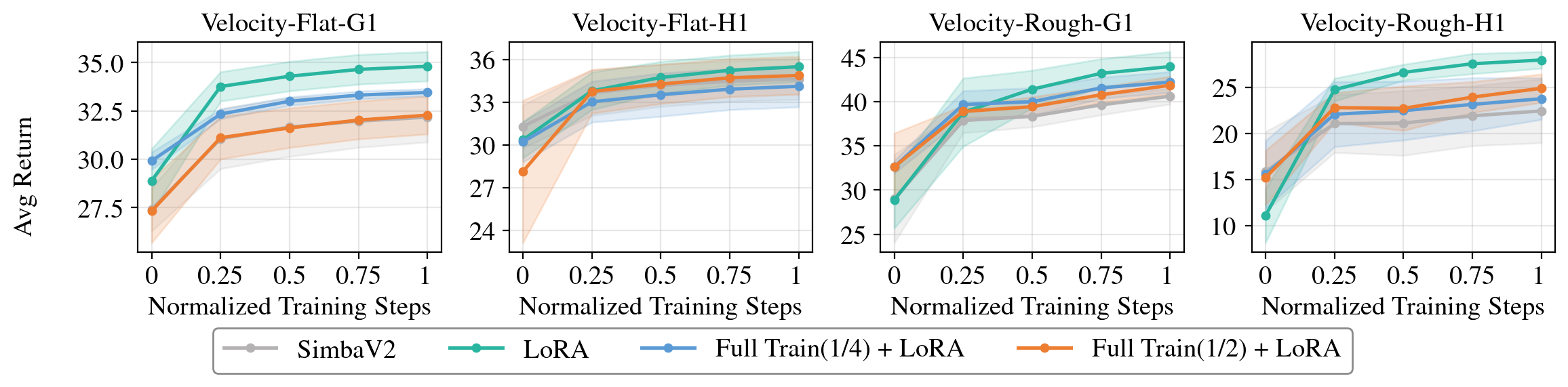}
    \caption{Comparison of pure LoRA training and hybrid schemes that perform dense updates for a fraction of training before switching to LoRA. Pure LoRA achieves competitive or superior performance, while additional dense pretraining provides only modest gains. Results are averaged over 5 seeds, with shaded regions indicating standard deviation.}
    \label{fig:lora_warm}
\end{figure}
% \begin{figure}[H]
%     \centering
%     \includegraphics[width=1.0\linewidth]{figures/sparse_lora_grid.png}
%     \caption{\textbf{Performance and parameter trade-off under critic scaling.}
%     Episode return versus the number of \emph{trainable} critic parameters on two challenging
%     DMC tasks at two model scales (\texttt{Base} and \texttt{Large}).
%     Shaded regions denote mean $\pm$ standard deviation over 3 seeds.}
%     \label{fig:param_tradeoff}
% \end{figure}

\section{Ablation Study}
\label{ap:ablation}
In this section, we conduct a series of ablation studies using the SimbaV2 backbone at two model scales: a smaller critic with $L=2$ and $d_h=512$, and a larger critic with $L=4$ and $d_h=1024$. These studies are designed to isolate the key design choices of our method, including the frozen base matrix, hyperspherical weight normalization, and low-rank parameterization.

\begin{figure}[h]
    \centering
    \includegraphics[width=1.0\linewidth]{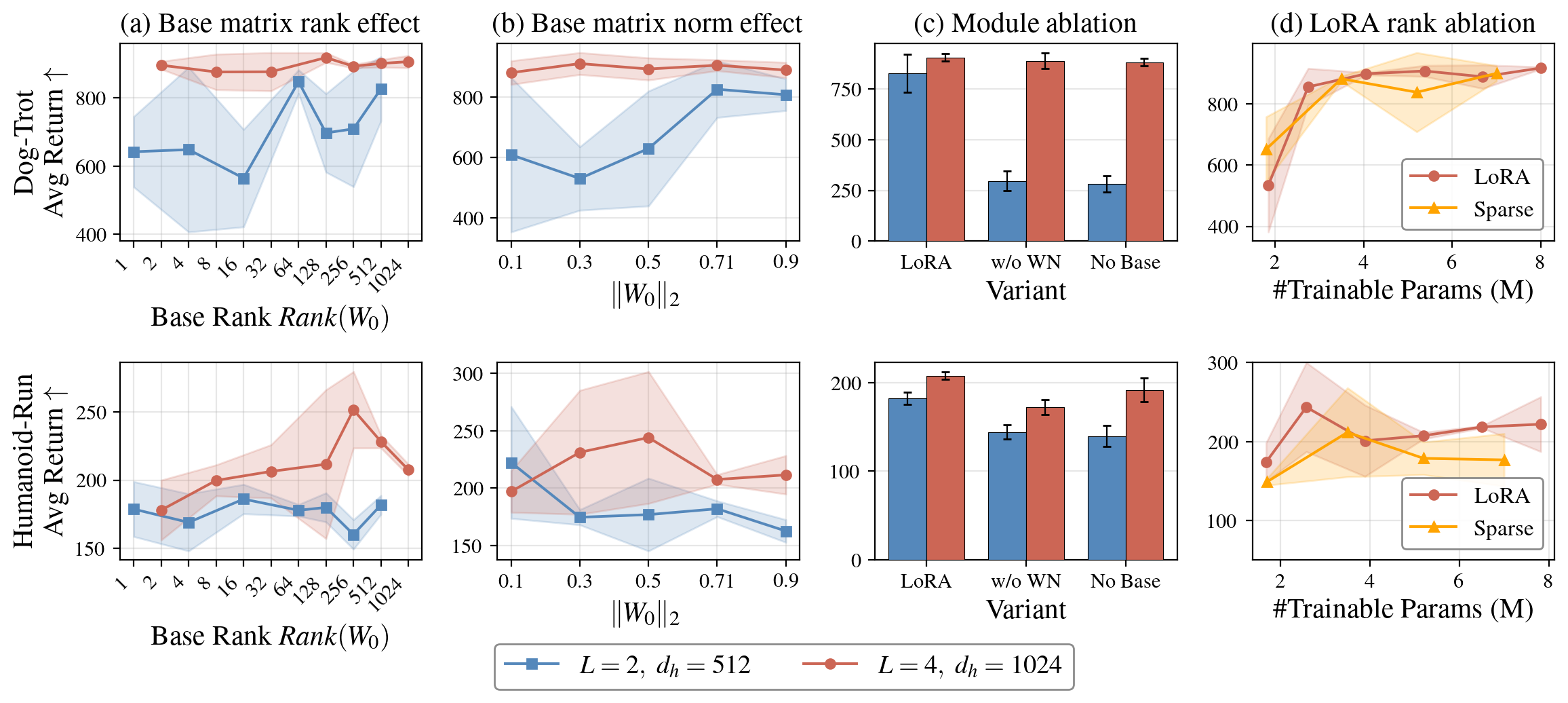}
    \caption{\textbf{Ablation study of the LoRA critic design.} We evaluate four key design choices on \texttt{Dog-Trot} (top) and \texttt{Humanoid-Run} (bottom): \textbf{(a)} the rank of the frozen base matrix \(W_0\), \textbf{(b)} the initialization norm \(\lVert w_0 \rVert_2\), \textbf{(c)} the effect of removing hyperspherical weight normalization or the frozen base matrix, and \textbf{(d)} the trainable-parameter trade-off obtained by varying the LoRA rank and comparing against static sparsity. Results are shown for two critic scales}
    \label{fig:ablation}
\end{figure}

\paragraph{Effect of the frozen base-matrix rank.}
Following~\cite{lee2025hyperspherical}, we initialize the frozen base matrix as an orthogonal matrix, which is full rank by default. To study the role of this fixed random basis, we vary the rank of the frozen base matrix while keeping the LoRA update rank unchanged. Figure~\ref{fig:ablation}(a) shows that LoRA is generally robust to this choice, especially with the larger critic. The smaller critic is more sensitive to the base-matrix rank, particularly on \texttt{Dog-Trot}, suggesting that the fixed base representation becomes more important when critic capacity is limited. In contrast, for deeper critics with residual connections, performance is much less affected by the base-matrix rank, likely because the architecture already provides sufficient representational capacity. Overall, these results indicate that our method does not rely on a narrowly tuned base-matrix rank.

\paragraph{Effect of the base-matrix norm.}
Figure~\ref{fig:ablation}(b) examines how the initialization scale of the frozen base-matrix norm $\|w_0\|_2$ affects LoRA training. Our default setting is $\kappa=\|w_0\|_2^2=0.5$, corresponding to $\|w_0\|_2=\sqrt{0.5}$. Overall, the larger critic achieves higher returns than the smaller critic and is less sensitive to this choice. The strongest sensitivity appears on \texttt{Humanoid-Run}: the larger critic performs best with intermediate norms, whereas the smaller critic favors smaller norms and degrades as $\|w_0\|_2$ increases. These results suggest that the initialization scale of the frozen backbone interacts with both task difficulty and critic capacity. In practice, intermediate sub-unit norms provide sufficient headroom for the LoRA residual and tend to yield stronger RL performance.

\paragraph{Effect of module design.}
\label{ap:module_design}
Figure~\ref{fig:ablation} (c) compares our proposed normalized LoRA against two ablations: \emph{LoRA-NoWN}, which removes the
weight normalization mechanism introduced in~\eqref{eq:implicit_s_constraint}, and \emph{NoBase}, which parameterizes the critic as $W = BA$ without a
frozen backbone $W_0$. Introducing weight normalization consistently
improves RL performance and robustness across tasks and model sizes; the
effect is most pronounced for the smaller critic,
where removing either component leads to a substantial drop in return. For the larger critic, the three variants perform more similarly, and normalized LoRA remains among the strongest variants, although the gap becomes smaller.

%These results indicate that the frozen base matrix and weight normalization are complementary: the former supplies a stable feature basis, while the latter keeps the low-rank update well-conditioned, together stabilizing critic optimization and enabling effective large-scale reinforcement learning.

\paragraph{Parameter efficiency at matched budget.}
Figure~\ref{fig:ablation} (d) fixes the critic architecture at \(L{=}4, d_h{=}1024\) and varies the LoRA rank \(r \in \{16, 32, 64, 96, 128, 160\}\), plotting final return against the number of trainable parameters. As a controlled baseline, we compare against Sparse~\cite{ma2025sparsity}, which achieves a similar parameter budget by pruning a dense critic rather than restricting updates to a low-rank subspace. We observe that LoRA remains broadly robust across ranks in both settings and outperforms Sparse on both tasks over most of the parameter range we evaluate. 
These results suggest that, under a fixed capacity budget, structured low-rank updates provide a more effective parameterization than unstructured sparsity.

\section{Conclusion}

In this work, we study Low-Rank Adaptation (LoRA) as a structural regularizer for critic learning in off-policy reinforcement learning. By freezing randomly initialized base matrices and optimizing only low-rank adapters, our method constrains critic updates to a low-dimensional subspace, helping mitigate overfitting in replay-based bootstrapped training. Empirically, across SAC and FastTD3 with SimbaV2 critics, as well as SAC with BRC critics, our method consistently reduces critic loss and improves policy performance over strong full-parameter and sparse baselines. These results suggest that controlling critic overfitting is itself an effective principle for improving off-policy RL, and opens a broader direction for designing architectures and training schemes that explicitly regularize value learning under non-stationary replay data.
\begin{ack}
We gratefully acknowledge the generous support of a
Qualcomm gift fund for this project.

Jonathan Petit is affiliated with Qualcomm, which provided a gift fund supporting this work. The authors declare no other competing interests.
\end{ack}
\newpage

\bibliographystyle{unsrt}  
\bibliography{reference}

\clearpage
\appendix
\section{Details of the Motivating Toy Experiment}
\label{app:toy_example}

\subsection{Environment}

The MDP is a 15-state chain $\mathcal{S} = \{0, \ldots, 14\}$ with two
actions $\mathcal{A} = \{0\text{ (left)}, 1\text{ (right)}\}$.
Transitions are stochastic: action $a$ succeeds with probability
$p = 0.9$, moving the agent one step in the intended direction, and
fails (the agent stays) with probability $1 - p = 0.1$.
States are reflected at both endpoints.
The reward is $R(s, a) = p \cdot \mathbf{1}[s + 1 = 14]$ for the
right action from state 13, and zero otherwise; the discount factor is
$\gamma = 0.97$.

The \textbf{behaviour policy} $\mu$ selects each action uniformly at
random.  The \textbf{target policy} $\pi$ always selects \textit{right}.
The distributional shift between $\mu$ and $\pi$ is the source of
off-policy bias: the replay buffer reflects the random-walk stationary
distribution of $\mu$ (concentrated near the centre of the chain),
while the evaluation distribution $d^{\pi}$ is concentrated near the
goal state.

The true value function $Q^{\pi}$ is computed analytically by solving
the Bellman system
$(I - \gamma P^{\pi})V^{\pi} = R^{\pi}$ and then
$Q^{\pi}(s, a) = R(s, a) + \gamma \sum_{s'} P(s'|s,a) V^{\pi}(s')$.

\subsection{Feature Map}

Both the dense and LoRA critics operate on a shared fixed random
feature map
\begin{equation}
    \phi(s, a) = \tanh\!\left(W_\phi\, e_{(s,a)}\right)
    \in \mathbb{R}^{D},\quad D = 64,
\end{equation}
where $e_{(s,a)} \in \mathbb{R}^{|\mathcal{S}|+|\mathcal{A}|}$ is the
concatenation of one-hot encodings of $s$ and $a$, and
$W_\phi \in \mathbb{R}^{D \times (|\mathcal{S}|+|\mathcal{A}|)}$ is
drawn once from $\mathcal{N}(0, 1/\sqrt{D})$ and frozen throughout.
The feature table is precomputed for all $(s, a)$ pairs.

\subsection{Critic Architectures}

The critic is a two-hidden-layer MLP with hidden width $H = 256$ and
ReLU activations:
\begin{equation}
    Q_\theta(s,a)
    = w_{\mathrm{out}}^\top \,\mathrm{ReLU}\!\left(
        W_1^{\mathrm{eff}}\, \mathrm{ReLU}\!\left(
            W_0^{\mathrm{eff}}\, \phi(s,a)
        \right)
    \right),
\end{equation}
with $W_0 \in \mathbb{R}^{H \times D}$,
$W_1 \in \mathbb{R}^{H \times H}$, and
$w_{\mathrm{out}} \in \mathbb{R}^{H}$.
No bias terms are used.

\paragraph{Dense parameterisation.}
$W_0^{\mathrm{eff}} = W_0$ and $W_1^{\mathrm{eff}} = W_1$ are both
trained freely; $w_{\mathrm{out}}$ is also trained.
Total trainable parameters: $HD + H^2 + H$.

\paragraph{LoRA-$r$ parameterisation.}
Both weight matrices are decomposed as
\begin{equation}
    W_i^{\mathrm{eff}} = W_i^{\mathrm{base}} + B_i A_i,
    \qquad i \in \{0, 1\},
\end{equation}
where $W_i^{\mathrm{base}}$ is frozen at its initialisation value
(identical to the dense critic's starting point), and
$A_i$, $B_i$ are the trainable low-rank factors with
$A_0 \in \mathbb{R}^{r \times D}$,
$A_1 \in \mathbb{R}^{r \times H}$,
$B_0, B_1 \in \mathbb{R}^{H \times r}$.
Following standard LoRA initialisation, $A_i$ is drawn from
$\mathcal{N}(0, 1/\sqrt{\mathrm{fan\_in}})$ and $B_i$ is initialised
to zero, so $W_i^{\mathrm{eff}} = W_i^{\mathrm{base}}$ at step~0 and
both critics produce identical predictions at initialisation.
The output layer $w_{\mathrm{out}}$ is trained in both cases.
The rank is swept over $r \in \{1, 2, 4, 8, 16, 32, 64, 128, 256\}$.

\subsection{Training Regimes}

\paragraph{Static regression.}
At each step a mini-batch of $(s, a)$ pairs is drawn uniformly from
$\mathcal{S} \times \mathcal{A}$.
The regression target is the precomputed true value
$y = Q^{\pi}(s, a)$, and the loss is
$\mathcal{L} = \frac{1}{2}\,\mathbb{E}\bigl[(Q_\theta(s,a) - y)^2\bigr]$.
There is no bootstrapping and no distributional shift in this regime;
it serves as a controlled baseline to measure the pure capacity effect
of the rank constraint.

\paragraph{Bootstrapped off-policy TD.}
Transitions $(s, a, r, s')$ are sampled from a fixed replay buffer of
500 transitions collected under $\mu$.
The TD target is
\begin{equation}
    y = r + \gamma\,\bar{Q}(s',\,\pi(s')),
\end{equation}
where $\bar{Q}$ is a \emph{target network}---a slow-moving copy of
$Q_\theta$ updated via Polyak averaging
$\bar{\theta} \leftarrow (1 - \tau)\,\bar{\theta} + \tau\,\theta$
with $\tau = 0.02$.
The action used in the bootstrap is always $\pi(s')$ (the target
policy), even though the stored transition was generated under $\mu$,
which is the source of the off-policy distributional shift.
The loss is $\mathcal{L} = \frac{1}{2}\,\mathbb{E}\bigl[(Q_\theta(s,a) - \mathrm{sg}(y))^2\bigr]$,
where $\mathrm{sg}(\cdot)$ denotes stop-gradient.

\subsection{Evaluation Metrics}

All metrics are computed every 300 training steps.

\begin{itemize}[leftmargin=*,nosep]
    \item \textbf{True-$Q$ error on $d^{\pi}$:}
        $\varepsilon_Q = \sqrt{\sum_s d^{\pi}(s)\,\bigl(Q_\theta(s,\pi(s)) - Q^{\pi}(s,\pi(s))\bigr)^2}$.
        This is the primary metric; it measures critic quality
        under the target policy's state distribution.

    \item \textbf{Bellman residual (RMS):}
        $\varepsilon_B = \sqrt{\frac{1}{|\mathcal{S}||\mathcal{A}|}\sum_{s,a}\bigl(Q_\theta(s,a) - \mathcal{T}^{\pi}Q_\theta(s,a)\bigr)^2}$,
        where $\mathcal{T}^{\pi}Q(s,a) = R(s,a) + \gamma\sum_{s'} P(s'|s,a)\,Q(s',\pi(s'))$
        is the exact Bellman operator.
        A low Bellman residual means the critic is self-consistent with
        its own Bellman equation; note that this is computed with the
        true transition model and is \emph{not} the training loss.
\end{itemize}

\subsection{Hyperparameters}
The hyperparameters are listed in Table~\ref{tab:toy_hparams}.

\begin{table}[h]
    \centering
    \caption{Hyperparameters for the toy motivating experiment.}
    \label{tab:toy_hparams}
    \begin{tabular}{lc}
        \toprule
        Hyperparameter & Value \\
        \midrule
        \multicolumn{2}{l}{\textit{MDP}} \\
        Number of states $|\mathcal{S}|$ & 15 \\
        Number of actions $|\mathcal{A}|$ & 2 \\
        Transition success probability $p$ & 0.9 \\
        Discount factor $\gamma$ & 0.97 \\
        \midrule
        \multicolumn{2}{l}{\textit{Feature map \& critic}} \\
        Feature dimension $D$ & 64 \\
        Hidden width $H$ & 256 \\
        Activation & ReLU \\
        \midrule
        \multicolumn{2}{l}{\textit{Replay \& policies}} \\
        Replay buffer size & 500 \\
        Behaviour policy $\mu$ & Uniform random \\
        Target policy $\pi$ & Always right \\
        \midrule
        \multicolumn{2}{l}{\textit{Optimisation}} \\
        Optimiser & Adam \\
        Learning rate & $10^{-3}$ \\
        Batch size & 64 \\
        Training steps & 12\,000 \\
        Target network Polyak rate $\tau$ & 0.02 \\
        Evaluation interval & 300 steps \\
        \midrule
        \multicolumn{2}{l}{\textit{LoRA}} \\
        Rank sweep $r$ & $\{1,2,4,8,16,32,64,128,256\}$ \\
        $A_i$ initialisation & $\mathcal{N}(0, 1/\sqrt{\mathrm{fan\_in}})$ \\
        $B_i$ initialisation & $0$ \\
        \midrule
        Random seeds & 5 \\
        \bottomrule
    \end{tabular}
\end{table}

\section{Hyperspherical Weight Normalization with LoRA}
\label{sec:ab_norm}
This section applies only to the SAC+SimbaV2 experiments. The SAC with SimbaV2 reference implementation performs post-update hyperspherical projection of linear weights while FastTD3 with SimbaV2 and SAC with BRC do not use this post-update weight projection. The same projection could be used with different backbones that impose row-wise hyperspherical weight constraints, but it is not used for the FastTD3 and BRC configurations evaluated here.

\begin{figure}[t]
    \centering
    \includegraphics[width=0.50\linewidth]{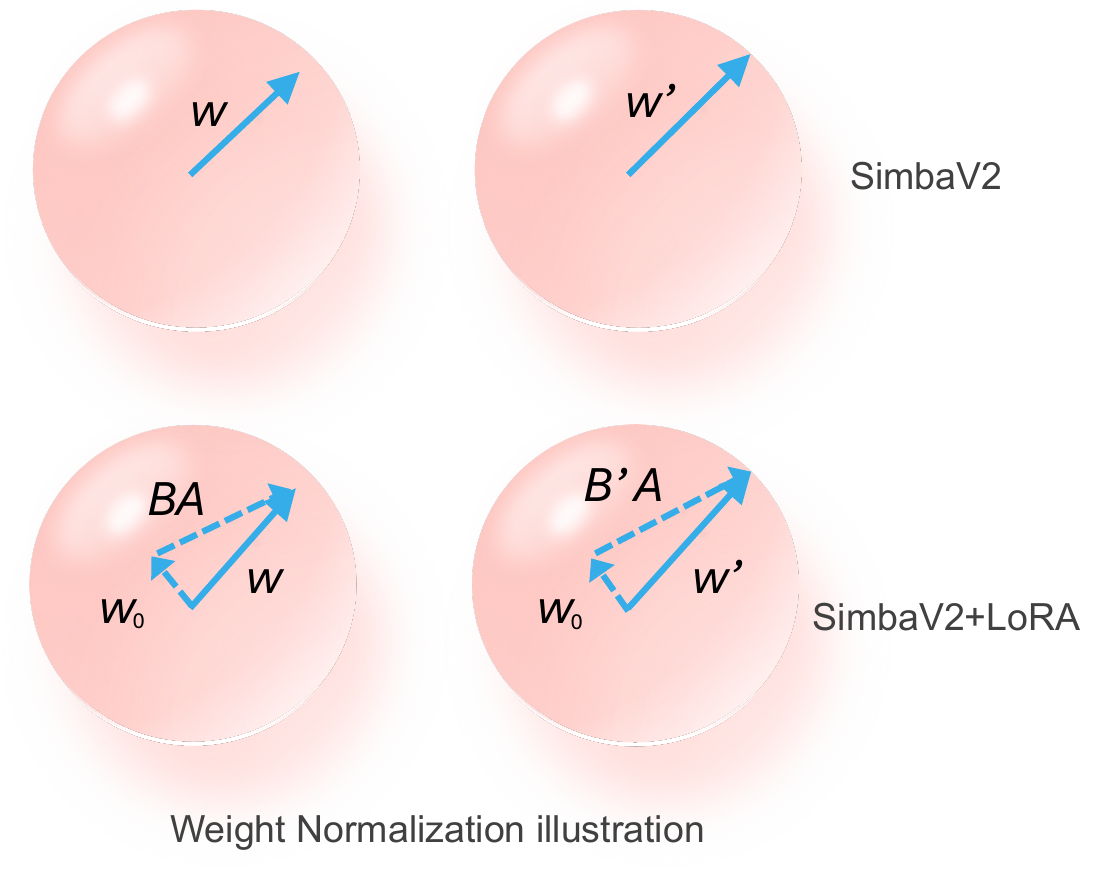}
    \caption{\textbf{LoRA-compatible hyperspherical projection.}
    Top: standard SimbaV2 projects each updated weight vector \(W\) onto the unit
    hypersphere to obtain \(W'\). Bottom: in the LoRA parameterization, the base
    vector \(W_0\) is frozen, so directly normalizing \(W_0+\Delta W\) would also
    rescale the frozen base. Instead, for the SAC+SimbaV2 setting, we solve for a
    scalar \(s_j>0\) and absorb it into the LoRA update so that
    \(W_0+s_j\Delta W\) lies on the unit hypersphere while \(W_0\) remains fixed.}
    \label{fig:weight_norm}
\end{figure}

SimbaV2 uses hyperspherical weight normalization, projecting each row of a weight matrix to unit norm after every update. Under LoRA, maintaining this constraint is non-trivial (see Appendix~\ref{ap_weight_norm}), because the base matrix $W_0$ is frozen and cannot be directly renormalized. To address this, we propose an effective-weight projection rule for LoRA. After each update to $(B,A)$, we form the residual $\Delta W$ in Eq.~\eqref{eq:lora_eff_weight} and solve Eq.~\eqref{eq:implicit_s_constraint} for a per-output scaling factor $s_j$. This keeps the effective row on the unit hypersphere while leaving the frozen base row unchanged:
\begin{equation}
\label{eq:implicit_s_constraint}
\big\| [W_0]_{j,:} + s_j [\Delta W]_{j,:} \big\|_2 \;=\; 1.
\end{equation}
Let \(w_j = [W_0]_{j,:}\) and \(\delta_j = [\Delta W]_{j,:}\) denote the \(j\)-th rows of \(W_0\) and \(\Delta W\), respectively. Squaring
Eq.~\eqref{eq:implicit_s_constraint} yields a quadratic equation in $s_j$:
\begin{equation}
\label{eq:implicit_standard_form}
(\|\delta_j\|_2^2) s_j^2 + 2\langle w_j,\delta_j\rangle s_j + (\|w_j\|_2^2 - 1) \;=\; 0.
\end{equation}
When $\|\delta_j\|_2>0$, the solutions are
\begin{equation}
\label{eq:s_closed_form_general}
s_j
=
\frac{-\langle w_j,\delta_j\rangle \;\pm\;
\sqrt{\langle w_j,\delta_j\rangle^2 - \|\delta_j\|_2^2(\|w_j\|_2^2-1)}}
{\|\delta_j\|_2^2},
\end{equation}

In practice, we clamp the denominator by a small \(\varepsilon\) only for numerical stability. In our implementation, we select the positive root $s_j > 0$ so that the LoRA update preserves its direction throughout training.

The closed-form solution in Eq.~\eqref{eq:s_closed_form_general} depends explicitly on both the norm of the base matrix row, \(\|w_j\|_2\), and the squared norm of the LoRA update, \(\|\delta_j\|_2^2\). This observation motivates the initialization scheme for the frozen base matrix and the LoRA adapters described below.

\section{Why standard row-wise weight normalization is incompatible with frozen LoRA bases}
\label{ap_weight_norm}

A natural approach to combining LoRA with hyperspherical weight normalization is to apply the standard row-wise normalization directly to the effective weight,
\begin{equation}
    W_{\mathrm{eff}}
    \leftarrow
    \frac{W_0 + \Delta W}{\|W_0 + \Delta W\|_2 + \varepsilon},
\end{equation}
where the normalization is applied independently to each row. Although this projects the effective weight onto the unit hypersphere in one step, it is incompatible with the frozen-base parameterization used by LoRA.
To see this, consider a single row and let
\[
w_j = [W_0]_{j,:}, \qquad \delta_j = [\Delta W]_{j,:}.
\]
Then standard row-wise normalization gives
\begin{equation}
    w_j + \delta_j
    \;\mapsto\;
    \frac{w_j + \delta_j}{\|w_j + \delta_j\|_2 + \varepsilon}
    \;=\;
    c_j w_j + c_j \delta_j,
\end{equation}
where
\begin{equation}
    c_j = \frac{1}{\|w_j + \delta_j\|_2 + \varepsilon}.
\end{equation}
Thus, both the frozen base component and the LoRA update are rescaled by the same factor \(c_j\). In particular, the base row is implicitly changed from \(w_j\) to \(c_j w_j\), which means that \(W_0\) is no longer truly frozen whenever \(c_j \neq 1\).
This creates two problems. First, it violates the intended LoRA assumption that the base matrix remains fixed throughout training. Second, it silently breaks the additive low-rank decomposition
\begin{equation}
    W_{\mathrm{eff}} = W_0 + \frac{\alpha_{\mathrm{lora}}}{r}BA,
\end{equation}
because after normalization the effective row can no longer be interpreted as the sum of the original frozen base row and a pure low-rank update. Instead, the normalization couples the frozen component and the trainable component through a shared multiplicative factor.

Our normalization rule avoids this issue by solving, for each row, a scalar \(s_j\) that is absorbed only into the LoRA update. Concretely, we enforce the unit-norm constraint on
\begin{equation}
    w_j + s_j \delta_j,
\end{equation}
while keeping \(w_j\) literally unchanged. In implementation, this rescaling is absorbed into the LoRA factor \(B\), so the frozen base matrix \(W_0\) remains untouched at every optimization step and the low-rank structure is preserved throughout training.

\section{Training Hyperparameters}
\label{ap:hyperparams}
Tables~\ref{tab:hyperparams_sac},~\ref{tab:hyperparams_fasttd3}, and~\ref{tab:hyperparams_brc} summarize the training hyperparameters for the SAC+SimbaV2 DMC-Hard experiments, the FastTD3+SimbaV2 IsaacLab experiments, and the SAC+BRC BroNet experiments, respectively. 
Within each setting, the dense, sparse, and LoRA variants use the same environment, replay buffer, batch size, update schedule, optimizer family, and learning-rate schedule, following the corresponding reference implementations~\cite{lee2025hyperspherical,ma2025sparsity,seo2025fasttd3,nauman2025brc,nauman2024bigger}. 
%In particular, SAC+SimbaV2 uses Adam, FastTD3+SimbaV2 uses AdamW with zero weight decay, and SAC+BRC follows the single-task BRC training setup with a BroNet. 
We note that, in our setting, BRC follows the single-task training setup described in the original BRC paper.
%The sparse baselines differ from their dense counterparts only by applying static one-shot random pruning at initialization, while the LoRA variants freeze the dense critic backbone weights and train low-rank adapter parameters together with the remaining non-LoRA critic parameters.

For our LoRA-based approach, all trainable parameters outside the LoRA modules use the same optimizer and learning-rate schedule as the corresponding SimbaV2 and BRC baselines. For the LoRA modules, we use AdamW with weight decay applied to the trainable factors \(A\) and \(B\). In the SAC+BRC and FastTD3+SimbaV2 configurations, which do not use post-update hyperspherical weight projection, this regularization is intended to discourage the residual \(BA\) from becoming large relative to the frozen base \(W_0\). In SAC+SimbaV2, a separate post-update projection enforces unit norm on the effective weight vectors.

\begin{table}[t]
\centering
\caption{SAC+SimbaV2 training hyperparameters (DMControl). All methods share the same base configuration from SimbaV2~\cite{lee2025hyperspherical}. The bottom section lists LoRA-specific parameters that apply only to the adapter factors.}
\label{tab:hyperparams_sac}
\small
\begin{tabular}{l l}
    \toprule
    \multicolumn{2}{c}{\textit{Shared across all methods (SimbaV2, Sparse, LoRA)}} \\
    \midrule
    Optimizer (base params) & Adam \\
    Learning rate (start $\to$ end) & $1 \times 10^{-4} \to 5 \times 10^{-5}$ \\
    LR schedule & Linear decay \\
    Batch size & 256 \\
    UTD (updates-to-data) ratio & 2 (action repeat = 2) \\
    Replay buffer size & $1 \times 10^{6}$ \\
    Warmup steps (min buffer) & 5{,}000 \\
    Discount $\gamma$ & Heuristic from TD-MPC2 ($\approx 0.99$) \\
    Target EMA $\rho$ & 0.005 \\
    Temperature init & 0.01 \\
    Target entropy coeff & $-0.5$ \\
    Observation normalization & Running statistics \\
    Reward normalization & Yes \\
    Normalized $G_{\max}$ & 5.0 \\
    Critic value head & Categorical ($n_{\text{atom}} = 101$, support $[-5, 5]$) \\
    Feature $\ell_2$ normalization & Yes \\
    Post-update weight hyperspherical projection & Yes \\
    Actor architecture & 1 block, $d_h = 128$ \\
    Total env steps & $1 \times 10^{6}$ \\
    $n$-step returns & 1 \\
    Double-Q & No \\
    \midrule
    \multicolumn{2}{c}{\textit{Sparse-specific}} \\
    \midrule
    Sparsity ratio ($d_h\!=\!512$, $L=2$ blocks) & 0.7 \\
    Sparsity ratio ($d_h\!=\!1024$, $L=4$ blocks) & 0.85 \\
    Pruning method & One-shot random (Erd\H{o}s--R\'{e}nyi) \\
    \midrule
    \multicolumn{2}{c}{\textit{LoRA-specific (adapter factors $A$, $B$ only)}} \\
    \midrule
    Optimizer (LoRA params) & AdamW \\
    Weight decay & $2 \times 10^{-4}$ \\
    LoRA rank $r$ & 96 \\
    LoRA scaling $\alpha_{\mathrm{lora}}$ & 96 \\
    $A$ initialization & Normal \\
    $B$ initialization & Normal \\
    base matrix init scale $\kappa$ & 0.5 \\
    LoRA applied to & Critic residual blocks only \\
    base matrix $W_0$ & Frozen (zero gradient) \\
    \bottomrule
\end{tabular}
\end{table}

\begin{table}[t]
\centering
\caption{FastTD3 training hyperparameters (IsaacLab). All methods share the same base configuration from FastTD3~\cite{seo2025fasttd3} with the SimbaV2 backbone. Task-specific defaults (e.g., $n$-step, $v_{\min}$/$v_{\max}$, total steps) follow the FastTD3 reference implementation.}
\label{tab:hyperparams_fasttd3}
\small
\begin{tabular}{l l}
    \toprule
    \multicolumn{2}{c}{\textit{Shared across all methods (SimbaV2, Sparse, LoRA)}} \\
    \midrule
    Optimizer (base params) & AdamW \\
    Weight decay (base params) & 0.0 \\
    Learning rate (start $\to$ end) & $3 \times 10^{-4} \to 3 \times 10^{-5}$ (cosine) \\
    Batch size & 8{,}192 \\
    Number of parallel envs & 4{,}096 \\
    Update interval & 4 env steps \\
    Gradient updates per interval & 16 (locomotion) / 8 (manipulation) \\
    Replay buffer size & $3{,}072$ transitions per env \\
    $n$-step returns & 8 (locomotion) / 1 (manipulation) \\
    Discount $\gamma$ & 0.99 \\
    Target EMA $\tau$ & 0.1 \\
    Exploration noise $\sigma$ & $[0.001, 0.4]$ (linear anneal) \\
    Action bounds & $[-1, 1]$ \\
    Observation normalization & Running statistics \\
    Critic value head & Categorical ($n_{\text{atom}} = 251$, see below) \\
    Feature $\ell_2$ normalization & Yes \\
    Post-update weight hyperspherical projection & No \\
    Actor architecture & 1 block, $d_h = 256$ \\
    Critic architecture & 4 blocks, $d_h = 1024$ \\
    Mixed-precision (AMP) & Yes (bf16) \\
    Double-Q & Yes \\
    \midrule
    \multicolumn{2}{c}{\textit{Task-specific value support}} \\
    \midrule
    Locomotion tasks & $v_{\min} = -10$, $v_{\max} = 10$ \\
    Manipulation tasks & $v_{\min} = -500$, $v_{\max} = 500$ \\
    \midrule
    \multicolumn{2}{c}{\textit{Sparse-specific}} \\
    \midrule
    Sparsity ratio ($d_h\!=\!1024$, $L=4$ blocks) & 0.85 \\
    Pruning method & One-shot random (Erd\H{o}s--R\'{e}nyi) \\
    \midrule
    \multicolumn{2}{c}{\textit{LoRA-specific (adapter factors $A$, $B$ only)}} \\
    \midrule
    Optimizer (LoRA params) & AdamW \\
    Weight decay & $6 \times 10^{-4}$\\
    LoRA rank $r$ & 96 \\
    LoRA scaling $\alpha_{\mathrm{lora}}$ & 96 \\
    $A$ initialization & Normal \\
    $B$ initialization & Normal \\
    base matrix init scale $\kappa$ & 0.5 \\
    LoRA applied to & Critic residual blocks only \\
    base matrix $W_0$ & Frozen (zero gradient) \\
    \bottomrule
\end{tabular}
\end{table}

\begin{table}[t]
\centering
\caption{SAC+BRC hyperparameters for DMC-Hard with BroNet. These runs follow single-task BRC~\cite{nauman2025brc} settings with the BroNet~\cite{nauman2024bigger} backbone}
\label{tab:hyperparams_brc}
\small
\begin{tabular}{l l}
\toprule
\multicolumn{2}{c}{\textit{Shared across BRC, Sparse-BRC, and LoRA-BRC}} \\
\midrule
Algorithm & SAC+BRC \\
Critic target aggregation & Ensemble average, no clipped double Q \\
Critic ensemble size & 2 \\
Optimizer & AdamW \\
Actor / critic learning rate & $3\times 10^{-4}$ \\
Weight decay & $1\times 10^{-4}$ \\
Temperature optimizer & Adam, learning rate $3\times 10^{-4}$, $\beta_1=0.5$ \\
Initial temperature & 0.1 \\
Target entropy & $-|\mathcal{A}|/2$ \\
Discount $\gamma$ & 0.99 \\
Target EMA $\tau$ & 0.005 \\
Batch size & 256 \\
UTD ratio & 2 \\
Replay buffer size & $10^6$ \\
Random action warmup & 5,000 steps \\
Total environment steps & $10^6$ \\
Evaluation episodes & 10 \\
Categorical critic & 101 atoms, support $[-10,10]$ \\
Reward normalization & BRC return normalization \\
Post-update weight hyperspherical projection & No \\
Actor architecture & 1 block, $d_h = 256$ \\
Critic architecture & 2 blocks, $d_h = 4096$ \\
\midrule
\multicolumn{2}{c}{\textit{Sparse-specific}} \\
\midrule
Critic sparsity & 0.90 \\
Pruning method & One-shot random (Erd\H{o}s--R\'{e}nyi) \\
\midrule
\multicolumn{2}{c}{\textit{LoRA-specific}} \\
\midrule
Optimizer (LoRA params) & AdamW \\
LoRA rank $r$ & 128 \\
LoRA scaling $\alpha_{\mathrm{lora}}$ & 128 \\
weight decay & $6\times 10^{-4}$ \\
$A$ initialization & Normal \\
$B$ initialization & Normal \\
base matrix init scale $\kappa$ & 0.5 \\
LoRA applied to & Critic residual blocks only \\
base matrix $W_0$ & Frozen (zero gradient) \\
\bottomrule
\end{tabular}
\end{table}

Table~\ref{tab:dmc_hard_list} describes the seven DMControl hard tasks used in our SAC experiments.
Table~\ref{tab:isaaclab_envs} describes the six IsaacLab tasks used in our FastTD3 experiments. These span locomotion (Unitree H1 and G1 robots on flat and rough terrain) and dexterous in-hand manipulation (Allegro and Shadow hands).

\begin{table}[h]
\centering
\caption{\textbf{DMC-Hard Complete List.} We evaluate a total of 7 continuous control tasks from the DMC-Hard benchmark. Below, we provide a list of all the tasks considered. The baseline performance for each task is reported at 1M environment steps.}
\label{tab:dmc_hard_list}
\begin{tabular}{l cc}
    \toprule
    \textbf{Task} & \textbf{Observation dim $|\mathcal{O}|$} & \textbf{Action dim $|\mathcal{A}|$} \\
    \midrule
    \texttt{dog-run} & 223 & 38 \\
    \texttt{dog-trot} & 223 & 38 \\
    \texttt{dog-stand} & 223 & 38 \\
    \texttt{dog-walk} & 223 & 38 \\
    \texttt{humanoid-run} & 67 & 24 \\
    \texttt{humanoid-stand} & 67 & 24 \\
    \texttt{humanoid-walk} & 67 & 24 \\
    \bottomrule
\end{tabular}
\end{table}

\begin{table}[h]
\centering
\caption{\textbf{IsaacLab Complete List.} We evaluate a total of 6 continuous control tasks from IsaacLab. Observation and action dimensions are determined by the robot morphology and task configuration. The baseline performance for each task is reported at 1M environment steps.}
\label{tab:isaaclab_envs}
\begin{tabular}{l cc}
    \toprule
    \textbf{Task} & \textbf{Observation dim $|\mathcal{O}|$} & \textbf{Action dim $|\mathcal{A}|$} \\
    \midrule
    \texttt{Isaac-Velocity-Flat-H1-v0} & 69 & 19 \\
    \texttt{Isaac-Velocity-Flat-G1-v0} & 123 & 37 \\
    \texttt{Isaac-Velocity-Rough-H1-v0} & 256 & 19 \\
    \texttt{Isaac-Velocity-Rough-G1-v0} & 310 & 37 \\
    \texttt{Isaac-Repose-Cube-Allegro-Direct-v0} & 124 & 16 \\
    \texttt{Isaac-Repose-Cube-Shadow-Direct-v0} & 157 & 20 \\
    \bottomrule
\end{tabular}
\end{table}

Table~\ref{tab:param_counts} reports the actor parameter count, critic total parameter count, and critic trainable parameter count for all three methods under different settings. Since FastTD3 uses double Q-networks whereas our SAC+SimbaV2 configuration uses a single Q-network, the critic parameter count in FastTD3 is approximately doubled even when both methods use the same SimbaV2 critic configuration, i.e., $d_h=1024$ and $L=4$.

\begin{table}[t]
\centering
\caption{
% Parameter counts for SimbaV2, Sparse-SimbaV2, and LoRA. Actor parameters are fully trainable for all methods. For SAC, we report counts using a representative Humanoid environment; for FastTD3, we use Isaac-Velocity-Flat-H1.
Parameter counts for dense, sparse, and LoRA across the SimbaV2 and BroNet+BRC settings. Actor parameters are unchanged across methods within each setting. Since tasks within each setting use the same SimbaV2 or BroNet hidden dimension and number of residual blocks, the parameter counts are approximately the same across environments, up to small differences caused by observation and action dimensions.
}
\label{tab:param_counts}
\resizebox{\linewidth}{!}{
\begin{tabular}{llcccc}
\toprule
Algorithm / Env. & Method & Critic Config. 
& Actor Params & Critic Total Params & Critic Trainable Params \\
\midrule
\multirow{6}{*}{SAC+SimbaV2 / Humanoid}
& SimbaV2 
& $d_h=512, L=2$ 
& 0.18M & 4.56M & 4.56M \\
& Sparse-SimbaV2 
& $d_h=512, L=2, s=0.70$ 
& 0.18M & 4.56M & 1.37M \\
& LoRA 
& $d_h=512, L=2, r=96$ 
& 0.18M & 5.54M & 1.35M \\
\cmidrule(lr){2-6}
& SimbaV2 
& $d_h=1024, L=4$ 
& 0.18M & 34.82M & 34.82M \\
& Sparse-SimbaV2 
& $d_h=1024, L=4, s=0.85$ 
& 0.18M & 34.82M & 5.22M \\
& LoRA 
& $d_h=1024, L=4, r=96$ 
& 0.18M & 38.75M & 5.20M \\
\midrule
\multirow{3}{*}{FastTD3 / Isaac-Velocity-Flat-H1}
& SimbaV2 
& $d_h=1024, L=4$ 
& 0.61M & 69.95M & 69.95M \\
& Sparse-SimbaV2 
& $d_h=1024, L=4, s=0.85$ 
& 0.61M & 69.95M & 10.53M \\
& LoRA 
& $d_h=1024, L=4, r=96$ 
& 0.61M & 77.81M & 10.70M \\
\midrule
\multirow{3}{*}{SAC+BRC / Humanoid}
& BroNet-BRC
& $d_h=4096, L=2$
& 0.16M & 135.89M & 135.89M \\
& Sparse-BroNet-BRC
& $d_h=4096, L=2, s=0.90$
& 0.16M & 135.89M & 13.59M \\
& LoRA-BroNet-BRC
& $d_h=4096, L=2, r=128$
& 0.16M & 144.28M & 10.06M \\
\bottomrule
\end{tabular}
}
\end{table}

\section{Detailed Experimental Results}
\label{ap:experiment}
\paragraph{Computational resources.}
All experiments were run on an internal Linux server with 8 NVIDIA RTX 6000 Ada Generation GPUs, each with 48GB of GPU memory. Each training run used one GPU unless otherwise stated. SAC experiments on DMC-Hard used CPU-based environment stepping and one GPU for network updates; FastTD3 experiments on IsaacLab used GPU-parallel simulation with 4,096 environments and bf16 mixed precision.

\paragraph{DMC-Hard results with SAC+SimbaV2.}
Table~\ref{tab:dmc_results} reports the final episode return for SAC+SimbaV2 on the seven DMC-Hard tasks. We compare dense SimbaV2, Sparse-SimbaV2 with static one-shot random pruning, and LoRA-SimbaV2 with frozen dense critic weights and trainable low-rank adapters. All numbers are computed over 5 seeds (seeds 0--4), and we report mean $\pm$ one standard deviation. Bold entries indicate the best mean return among the three methods for each task and scale.

\begin{table}[t]
\centering
\caption{Final episode return on DMControl hard locomotion tasks. Three methods are compared: SimbaV2 (full-parameter baseline), Sparse (random one-shot pruning~\cite{ma2025sparsity}), and LoRA (rank-96 adapters with frozen base matrices). Results are mean $\pm$ std over 5 seeds.}
\label{tab:dmc_results}
\small
\begin{tabular}{l ccc}
    \toprule
    \textbf{Task} & \textbf{SimbaV2} & \textbf{Sparse} & \textbf{LoRA (ours)} \\
    \midrule
    \multicolumn{4}{c}{\textit{$d_h=512$, $L=2$ blocks}} \\
    \midrule
        dog-run & \textbf{554.9} $\pm$ 96.4 & 348.8 $\pm$ 110.3 & 504.4 $\pm$ 56.8 \\
        dog-trot & \textbf{864.3} $\pm$ 37.5 & 621.3 $\pm$ 233.7 & 837.4 $\pm$ 68.1 \\
        dog-stand & \textbf{969.8} $\pm$ 12.2 & 950.3 $\pm$ 22.9 & 955.7 $\pm$ 22.5 \\
        dog-walk & \textbf{926.7} $\pm$ 26.2 & 848.9 $\pm$ 46.7 & 908.2 $\pm$ 25.0 \\
        humanoid-run & 182.4 $\pm$ 20.9 & 151.1 $\pm$ 9.5 & \textbf{221.3} $\pm$ 54.4 \\
        humanoid-stand & 814.9 $\pm$ 173.5 & 781.5 $\pm$ 133.1 & \textbf{861.8} $\pm$ 127.5 \\
        humanoid-walk & 594.3 $\pm$ 55.2 & 549.0 $\pm$ 41.2 & \textbf{620.3} $\pm$ 16.4 \\
    \midrule
        \textbf{Average} & 701.1 & 607.3 & \textbf{701.3} \\
    \midrule
    \multicolumn{4}{c}{\textit{$d_h=1024$, $L=4$ blocks}} \\
    \midrule
        dog-run & 570.4 $\pm$ 59.1 & 631.8 $\pm$ 82.2 & \textbf{658.6} $\pm$ 56.6 \\
        dog-trot & 901.3 $\pm$ 21.4 & 862.8 $\pm$ 97.8 & \textbf{911.1} $\pm$ 16.0 \\
        dog-stand & 957.0 $\pm$ 23.9 & 973.1 $\pm$ 9.2 & \textbf{978.7} $\pm$ 6.4 \\
        dog-walk & 935.1 $\pm$ 12.3 & \textbf{937.0} $\pm$ 12.4 & 936.1 $\pm$ 11.0 \\
        humanoid-run & 221.2 $\pm$ 53.6 & 200.8 $\pm$ 36.9 & \textbf{223.5} $\pm$ 22.2 \\
        humanoid-stand & 765.7 $\pm$ 100.9 & 890.1 $\pm$ 75.3 & \textbf{920.4} $\pm$ 28.1 \\
        humanoid-walk & 687.1 $\pm$ 78.0 & 622.9 $\pm$ 48.9 & \textbf{811.8} $\pm$ 104.6 \\
    \midrule
        \textbf{Average} & 719.7 & 731.2 & \textbf{777.2} \\
    \bottomrule
\end{tabular}
\end{table}

In the $d_h=512$, $L=2$ blocks settings, LoRA matches SimbaV2 on average (701.3 vs.\ 701.1) while substantially outperforming the sparse baseline (607.3).

In the $d_h=1024$, $L=4$ blocks settings, the benefit of low-rank regularization becomes clearer: LoRA achieves an average of 777.2, outperforming both SimbaV2 (719.7) and Sparse (731.2) by a substantial margin.
This widening gap at larger scale supports our central claim that constraining the update space via LoRA provides increasingly effective regularization as critic capacity grows and overfitting becomes more severe.

\paragraph{DMC-Hard results with SAC+BRC.}
Table~\ref{tab:brc_bronet_results} reports DMC-Hard evaluation using SAC+BRC with BroNet. 
This setting is included as an architecture general validation: unlike SimbaV2, BroNet uses LayerNorm-based residual blocks and does not rely on hyperspherical row projection.
Thus, improvements in this setting indicate that LoRA is not merely correcting a SimbaV2-specific parameterization.

\begin{table}[t]
\centering
\caption{Final episode return on DMC-Hard using SAC+BRC with a BroNet critic ($d_h=4096$, $L=2$). Results are mean $\pm$ standard deviation over 5 seeds. LoRA uses rank $r=128$ and $\alpha_{\mathrm{lora}}=128$.}
\label{tab:brc_bronet_results}
\small
\begin{tabular}{l ccc}
\toprule
\textbf{Task} & \textbf{BroNet-BRC} & \textbf{Sparse-BRC} & \textbf{LoRA-BRC} \\
\midrule
\texttt{dog-run} & $307.3 \pm 44.4$ & $375.0 \pm 68.3$ & $\mathbf{460.6 \pm 78.2}$ \\
\texttt{dog-trot} & $458.9 \pm 83.3$ & $\mathbf{810.8 \pm 52.2}$ & $773.8 \pm 104.4$ \\
\texttt{dog-stand} & $938.2 \pm 19.2$ & $961.0 \pm 16.2$ & $\mathbf{978.0 \pm 8.3}$ \\
\texttt{dog-walk} & $816.6 \pm 92.1$ & $868.9 \pm 39.0$ & $\mathbf{910.5 \pm 12.8}$ \\
\texttt{humanoid-run} & $199.6 \pm 53.1$ & $270.4 \pm 99.1$ & $\mathbf{330.8 \pm 63.5}$ \\
\texttt{humanoid-stand} & $871.4 \pm 118.0$ & $925.0 \pm 18.3$ & $\mathbf{930.3 \pm 15.2}$ \\
\texttt{humanoid-walk} & $783.3 \pm 100.6$ & $782.9 \pm 117.8$ & $\mathbf{901.5 \pm 25.0}$ \\
\midrule
% \textbf{Dog Avg.} & 630.3 & 753.9 & \textbf{780.7} \\
% \textbf{Humanoid Avg.} & 618.1 & 659.4 & \textbf{720.9} \\
\textbf{Average} & 625.1 & 713.4 & \textbf{755.1} \\
\bottomrule
\end{tabular}
\end{table}

LoRA-BRC achieves the best average return, improving from 625.1 for dense BroNet-BRC and 713.4 for Sparse-BRC to 755.1. 
It obtains the highest mean score on six of the seven tasks, with especially large gains on \texttt{dog-run}, \texttt{humanoid-run}, and \texttt{humanoid-walk}. 
Together with the SimbaV2 results, this suggests that low-rank critic adaptation acts as a general update-space regularizer rather than a mechanism tied to one specific model design.

% \paragraph{IsaacLab results.}
% Table~\ref{tab:isaaclab_results} reports the final episode return for each method on six IsaacLab tasks using the FastTD3 algorithm.
% All numbers are computed over 5 seeds (seeds 0--4). We report mean $\pm$ one standard deviation.
% Bold entries indicate the best mean return among the three methods for each task.

\paragraph{IsaacLab results with FastTD3+SimbaV2.}
Table~\ref{tab:isaaclab_results} reports the final episode return for FastTD3+SimbaV2 on six IsaacLab tasks. 
We compare dense SimbaV2, Sparse-SimbaV2, and LoRA-SimbaV2 using the same FastTD3 training setup and critic architecture.
All numbers are computed over 5 seeds (seeds 0--4), and we report mean $\pm$ one standard deviation.
Bold entries indicate the best mean return among the three methods for each task.

\begin{table}[t]
\centering
\caption{Final episode return on IsaacLab tasks with FastTD3. Critic architecture: $d_h=1024$, $L=4$ blocks. Results are mean $\pm$ std over 5 seeds.}
\label{tab:isaaclab_results}
\small
\begin{tabular}{l ccc}
    \toprule
    \textbf{Task} & \textbf{SimbaV2} & \textbf{Sparse} & \textbf{LoRA (ours)} \\
    \midrule
    \multicolumn{4}{c}{\textit{Locomotion}} \\
    \midrule
        Velocity-Flat-H1 & $34.6 \pm 1.4$ & $34.1 \pm 1.4$ & $\textbf{35.5} \pm 1.1$ \\
        Velocity-Flat-G1 & $32.2 \pm 1.3$ & $33.1 \pm 1.0$ & $\textbf{34.8} \pm 0.8$ \\
        Velocity-Rough-H1 & $22.4 \pm 3.5$ & $24.4 \pm 0.5$ & $\textbf{27.9} \pm 0.9$ \\
        Velocity-Rough-G1 & $40.6 \pm 0.9$ & $42.0 \pm 1.1$ & $\textbf{44.0} \pm 1.7$ \\
    \midrule
        \textbf{Locomotion Avg} & 32.5 & 33.4 & \textbf{35.6} \\
    \midrule
    \multicolumn{4}{c}{\textit{In-hand manipulation}} \\
    \midrule
        Repose-Allegro & $8022.1 \pm 473.0$ & $7703.1 \pm 791.6$ & $\textbf{8184.0} \pm 203.9$ \\
        Repose-Shadow & $\textbf{12163.2} \pm 202.1$ & $11811.1 \pm 192.1$ & $12087.1 \pm 116.3$ \\
    \midrule
        \textbf{Manipulation Avg} & 10092.7 & 9757.1 & \textbf{10135.6} \\
    \bottomrule
\end{tabular}
\end{table}

LoRA-SimbaV2 achieves the best average return on both IsaacLab locomotion and in-hand manipulation. 
On locomotion, LoRA improves the average return from 32.5 for dense SimbaV2 and 33.4 for Sparse-SimbaV2 to 35.6. 
On manipulation, LoRA obtains the best average return of 10135.6, while remaining competitive on both individual manipulation tasks. 
These results show that the benefit of low-rank critic adaptation transfers from SAC on CPU-stepped DMC-Hard tasks to FastTD3 with massively parallel GPU simulation.

Full per-environment learning curves for all three main settings are provided in Figure~\ref{fig:full_results}. All the curves are mean~$\pm$~std across 5 seeds; $x$-axis is normalized training steps.

\begin{figure}
    \centering
    \includegraphics[width=0.85\linewidth]{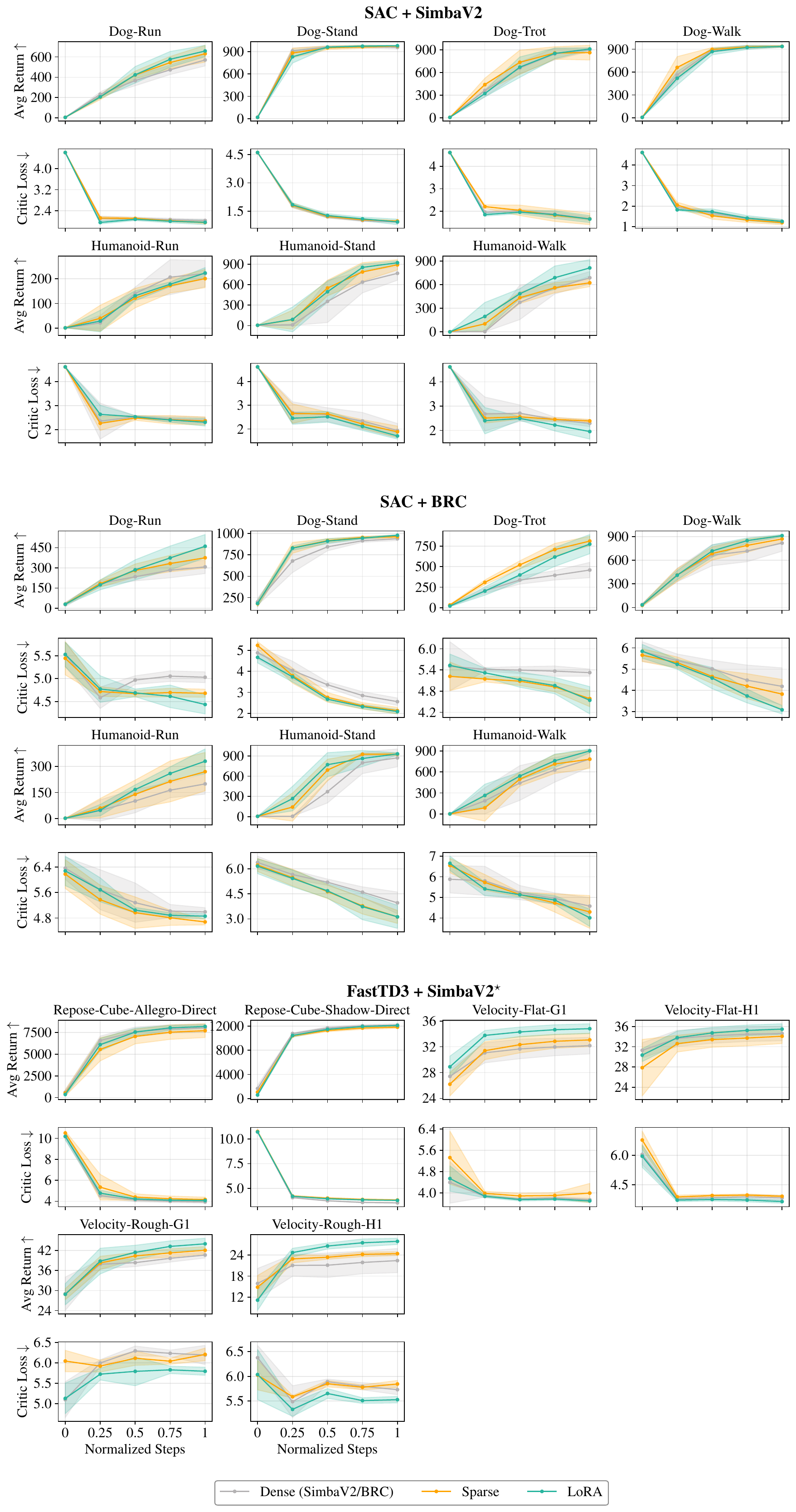}
    \caption{\textbf{Per-environment learning curves.}}
\label{fig:full_results}
\end{figure}

\paragraph{Computational cost analysis.}
During experiments, SimbaV2, Sparse, and LoRA exhibit comparable GPU VRAM usage and training time within the same algorithm. All methods share the same dense architecture, leading to nearly identical per-step computation. Training time is primarily dominated by environment interaction (CPU-based stepping for SAC on DMC; GPU-parallel rollouts for FastTD3 on IsaacLab) and large-batch updates of distributional critics~\cite{bellemare2017distributional}, rather than the number of trainable parameters. Similarly, VRAM consumption is largely driven by replay buffers and activation storage instead of the critic network parameters. Therefore, LoRA is not motivated by parameter or memory efficiency in RL, but by the regularization effect of constraining updates to a low-dimensional subspace, as demonstrated in our experiments. 

\section{Theoretical Proof}

\begin{lemma}[Existence of a positive scaling root under initialization]
\label{lem:positive_root_exists_app}
Consider the scaling equation in Eq.~\eqref{eq:s_closed_form_general} for row $j$, induced by the base weight $w_j$ and LoRA update $\delta_j$, with $\delta_j \neq 0$.
Under the initialization with $\|w_j\|_2^2 = \kappa$ with $\kappa \in (0,1)$, the corresponding quadratic equation in $s_j$ admits two real roots $s_j^{(0)}$ and $s_j^{(1)}$ satisfying
\[
s_j^{(0)} > 0,
\qquad
s_j^{(1)} < 0.
\]
In particular, there always exists a positive scaling root.
\end{lemma}

\begin{proof}
For each row $j$, the scaling coefficient $s_j$ is obtained by solving the quadratic equation associated with Eq.~\eqref{eq:s_closed_form_general}, which can be written in the form
\begin{equation}
\label{eq:quadratic_root_proof}
\|\delta_j\|_2^2\, s_j^2 + 2\langle w_j,\delta_j\rangle s_j + (\|w_j\|_2^2 - 1) = 0.
\end{equation}
Under our initialization, $\delta_j \neq 0$, so the quadratic coefficient satisfies
\[
\|\delta_j\|_2^2 > 0.
\]
Moreover, each base weight row is initialized to have norm
\[
\|w_j\|_2^2 = \kappa, \qquad \kappa \in (0,1),
\]
and therefore the constant term is
\[
\|w_j\|_2^2 - 1 = \kappa - 1 < 0.
\]

Since the product of the two roots of Eq.~\eqref{eq:quadratic_root_proof} is
\[
s_j^{(0)} s_j^{(1)}
=
\frac{\|w_j\|_2^2 - 1}{\|\delta_j\|_2^2}
=
\frac{\kappa - 1}{\|\delta_j\|_2^2}
< 0,
\]
the two roots must have opposite signs. Hence one root is positive and the other is negative. Denoting the positive root by $s_j^{(0)}$ and the negative root by $s_j^{(1)}$, we obtain
\[
s_j^{(0)} > 0,
\qquad
s_j^{(1)} < 0.
\]
This proves that a positive scaling root always exists.
\end{proof}

% Rebuttal additions shared by the conference and arXiv versions.
\clearpage
\section{Comparison with Dropout-Q}
\label{app:dropout_q}

We additionally compare with Dropout-Q, a representative critic-regularization method that applies dropout during critic training~\cite{hiraoka2022dropoutqfunctionsdoublyefficient}. For its implementation and hyperparameter settings, we follow the original paper and the accompanying reference code. We evaluate Dense, Sparse, LoRA, and Dropout-Q using the same five random seeds (5--9) on Dog Run with SAC+BRC and Velocity-Rough-G1 with FastTD3+SimbaV2.

As shown in Table~\ref{tab:dropout_q_returns}, Dropout-Q achieves a higher mean final return than the dense baseline on Velocity-Rough-G1. However, it does not provide a similar improvement on Dog Run, where its mean final return is slightly below that of the dense baseline. In contrast, LoRA achieves the highest mean final return among all four methods on both tasks.

\begin{table}[htbp]
    \centering
    \caption{\textbf{Final policy return with Dropout-Q as an additional baseline.} Results are evaluated at the final normalized training step and reported as mean $\pm$ standard deviation over five random seeds (5--9). Bold entries indicate the highest mean return for each task.}
    \label{tab:dropout_q_returns}
    \small
    \begin{tabular}{lcc}
        \toprule
        Method & Dog Run & Velocity-Rough-G1 \\
               & SAC+BRC & FastTD3+SimbaV2 \\
        \midrule
        Dense & $321.8 \pm 35.1$ & $40.73 \pm 3.39$ \\
        Sparse & $409.0 \pm 56.9$ & $42.14 \pm 4.96$ \\
        LoRA & $\mathbf{476.1 \pm 46.0}$ & $\mathbf{43.96 \pm 1.53}$ \\
        Dropout-Q & $311.5 \pm 45.5$ & $42.69 \pm 1.75$ \\
        \bottomrule
    \end{tabular}
\end{table}

\section{Limitations}
\label{app:limitations}

Our benchmark experiments focus on state-based continuous-control tasks, following the experimental settings of SimbaV2 and FastTD3. Although the proposed low-rank critic framework can be extended to other architectures and learning settings, we have not systematically evaluated pixel-based observations, offline reinforcement learning, or discrete-action domains. Evaluating these extensions remains an important direction for future work.

Our architecture-specific design ablations focus on SimbaV2, so the effects of LoRA rank and frozen-base initialization across other architectures remain to be systematically evaluated. Our experiments also do not fully disentangle the effects of low-rank parameterization and constrained optimization on the observed performance gains.

%%%%%%%%%%%%%%%%%%%%%%%%%%%%%%%%%%%%%%%%%%%%%%%%%%%%%%%%%%%%

\end{document}